\documentclass{article}%
\usepackage{amsfonts}
\usepackage{amsmath}
\usepackage{amssymb}
\usepackage{graphicx}%
\newtheorem{theorem}{Theorem}

\newtheorem{corollary}[theorem]{Corollary}

\newtheorem{definition}[theorem]{Definition}

\newtheorem{lemma}[theorem]{Lemma}

\newtheorem{proposition}[theorem]{Proposition}
\newtheorem{remark}[theorem]{Remark}

\newenvironment{proof}[1][Proof]{\noindent\textbf{#1.} }{\ \rule{0.5em}{0.5em}}
\begin{document}

\title{Hierarchical Neural Networks, $p$-Adic PDEs, and Applications to Image Processing}
\author{W. A. Z\'{u}\~{n}iga-Galindo\thanks{The first and third authors were partially
supported by the Lokenath Debnath Endowed Professorship.}\\University of Texas Rio Grande Valley\\School of Mathematical and Statistical Sciences\\One West University Blvd.,\\Brownsville, TX 78520, United States.\\wilson.zunigagalindo@utrgv.edu
\and B. A. Zambrano-Luna\\University of Alberta\\Department of Mathematical and Statistical Sciences\\CAB 632\\Edmonton, Alberta, Canada T6G 2G1\\bzambran@ualberta.ca
\and Baboucarr Dibba\\University of Texas Rio Grande Valley\\School of Mathematical and Statistical Sciences\\One West University Blvd.,\\Brownsville, TX 78520, United States.\\baboucarr.dibba01@utrgv.edu }
\maketitle

\begin{abstract}
The first goal of this article is to introduce a new type of p-adic
reaction-diffusion cellular neural network with delay. We study the stability
of these networks and provide numerical simulations of their responses. The
second goal is to provide a quick review of the state of the art of $p$-adic
cellular neural networks and their applications to image processing.

\end{abstract}

\section{Introduction}

This article has twofold purposes. First, it continues the investigation of
the first two authors on $p$-adic cellular neural networks (CNNs) and their
application to image processing; see \cite{Zambrano-Zuniga-1}%
-\cite{Zambrano-Zuniga-2}. The $p$-adic CNNs are hierarchical generalizations
of the classical CNNs introduced by Chuan and Yang in the 1980s; see
\cite{Chua-Yang}-\cite{Slavova}, and the references therein. The classical
CNNs are given by a set of integrodifferential equations controlling the
states of the neurons forming the network. Typically, the neurons are
organized in a lattice contained in $\mathbb{R}^{2}$. The $p$-adic integers,
$\mathbb{Z}_{p}$, forms an infinite rooted tree. The neurons are naturally
organized in layers (levels), and by cutting the tree $\mathbb{Z}_{p}$ at
level $l\geq1$, one obtains a finite rooted tree $G_{l}$ with $l$ levels and
$p^{l}$ vertices at the top level. In the $p$-adic CNNs the neurons are
organized in a finite rooted tree, $G_{l}$. In the limit when the number of
neurons tends to infinity, the $p$-adic CNNs admit limits that are abstract
evolution integrodifferential equations on $\mathbb{Z}_{p}$. This type of
equation can be studied using standard techniques for PDEs. In this article,
we studied $p$-adic reaction-diffusion CNNs with delay, see Section
\ref{Section_CNNS with delay}. We study the Cauchy problem and the stability
of the new networks. We also provide numerical simulations for the solutions
of these networks. Our numerical results show a chaotic behavior in the
responses of the $p$-adic CNNs with delay.

The second goal is to quickly review the state of the art in $p$-adic CNNs and
their applications to image processing. The $p$-adic CNNs are bioinspired by
the Wilson-Cowan models for the macroscopic activity patterns in the cortex of
mammals and invertebrates; see, e.g., \cite{Neural-Fields}, and the references
therein. The formulation of the discrete Wilson-Cowan models does not require
a particular topology for the neurons. For instance, these models are valid
for neural networks (NNs) with hierarchical topologies. However, when taking
the limit when the numbers of neurons tend to infinity, in most of the
available literature, it is assumed that neurons are organized in lattice
contained in $\mathbb{R}^{n}$; this is done to approximate certain sums by
Riemann--Stieltjes integrals.

Based on experimental data, it has been hypothesized that cortical neural
networks are arranged in fractal or self-similar patterns and have the
small-world property, see, e.g., \cite{Sporns et al 1}-\cite{Skoch et al}, and
the references therein. These properties are incompatible with the standard
Wilson-Cowan models formulated on $(\mathbb{R}^{n},+)$. To overcome this
problem, it is necessary to assume that the neurons are organized
hierarchically, which requires to change the additive group $(\mathbb{R}%
^{n},+)$ for an hierarchical additive group. In \cite{Zuniga-Entropy}, the
first two authors formulated a $p$-adic version of the Wilson-Cowan models on
the group $\left(  \mathbb{Z}_{p},+\right)  $, and showed that this model is a
good substitute of the standard one, where the cortical neural networks are
arranged in self-similar patterns and have the small-world property. The
$p$-adic CNNs are simplified versions of the $p$-adic Wilson-Cowan models that
have showed a great potential to perform several tasks in image processing.
Finally, we present \ some relevant research problems connecting $p$-adic NNs
with image processing and neuronal dynamics.

\section{Motivations and Preliminaries}

In this section, we discuss the motivations behind the study of hierarchical
biological and artificial NNs using methods of non-Archimedean analysis.

\subsection{PDE based models for neural fields}

The spatiotemporal continuum models for the dynamics of macroscopic activity
patterns in the cortex were introduced in the 1970s following the seminal work
by Wilson and Cowan, Amari, among others, see, e.g., \cite{Wilson-Cowan-1}%
-\cite{Amari}, see also \cite{Neural-Fields} and the references therein. Such
models take the form of integrodifferential evolution equations in which the
integration kernels represent the strength of the connection between neurons,
or more generally, the spatial distribution of synaptic connections between
different neural populations, and the state macroscopic variables represent
some average of neural activity. These PDEs based models have been used to
model phenomena such as short-term memory \cite{Laing}, the head direction
system \cite{Zhang}, visual hallucinations \cite{Ermentrout}-\cite{Ermentrout
et al}, and EEG rhythms \cite{Bojak}.

The simplest continuous model in one spatial dimension is%
\begin{equation}
\frac{\partial u\left(  x,t\right)  }{\partial t}=-u\left(  x,t\right)  +%
%TCIMACRO{\dint \nolimits_{-\infty}^{\infty}}%
%BeginExpansion
{\displaystyle\int\nolimits_{-\infty}^{\infty}}
%EndExpansion
w\left(  x-y\right)  f\left(  u\left(  y,t\right)  \right)  dy,
\label{Model_0}%
\end{equation}
where at the position $x\in\mathbb{R}$ there is a group of neurons, and
$u\left(  x,t\right)  $ is some average network activity at the time
$t\in\mathbb{R}_{+}:=\left\{  s\in\mathbb{R};s\geq0\right\}  $. The kernel $w$
gives the strength of connection between the neurons located at the position
$x$ and $y$. In the simplest model\ $w$ is symmetric, i.e. $w(x)=w(-x)$,
$\lim_{x\rightarrow\infty}w(x)=0$, $\int_{-\infty}^{\infty}w\left(  x\right)
dx<\infty$, and continuous. The non-linear function $f$ is the activation (or
firing) function. For instance, a non-decreasing function satisfying
$\lim_{s\rightarrow-\infty}f(s)=0$, $\lim_{s\rightarrow\infty}f(s)=1$. The
amount $f\left(  u\left(  x,t\right)  \right)  $ is the firing rate normalized
to one. The integral of $w\left(  x-y\right)  f\left(  u\left(  y,t\right)
\right)  $ over $y$ represents the influence of all neurons at position $y$ on
the neurons at position $x$.

There are two basic discrete models. The first one is the voltage-based rate
model:%
\begin{equation}
\tau_{m}\frac{\partial a_{i}\left(  t\right)  }{\partial t}=-a_{i}\left(
t\right)  +%
%TCIMACRO{\dsum \limits_{j}}%
%BeginExpansion
{\displaystyle\sum\limits_{j}}
%EndExpansion
w_{ij}f\left(  a_{j}\left(  t\right)  \right)  , \label{Model_1}%
\end{equation}
where $\tau_{m}$ is a time constant and $w_{ij}$ is the strenght between
neurons $i$ and $j$. The second one is called the activity-based model, it has
the form%
\begin{equation}
\tau_{s}\frac{\partial v_{i}\left(  t\right)  }{\partial t}=-v_{i}\left(
t\right)  +f\left(
%TCIMACRO{\dsum \limits_{j}}%
%BeginExpansion
{\displaystyle\sum\limits_{j}}
%EndExpansion
w_{ij}v_{j}\left(  t\right)  \right)  , \label{Model_2}%
\end{equation}
where $\tau_{s}$ is a time constant. The reader may consult \cite{Bressloff}
for an in-depth discussion about these two models. \textit{A fundamental
observation is that in these models, the topology of the neurons (the
architecture of the network) is arbitrary}. For instance, the neurons may be
organized in a hierarchical tree-like structure. The continuous versions of
the models (\ref{Model_1})-(\ref{Model_2}) are obtained by taking the limit
when the number of neurons tends to infinity. In the limit, the discrete
variable $i$ becomes a continuous variable $x$. Almost all the published
literature assumes that the neurons are organized in a lattice contained in
$\mathbb{R}^{n}$. Using this hypothesis, the sums in (\ref{Model_1}%
)-(\ref{Model_2}) become Riemann--Stieltjes integrals\ in the limit when the
number of neurons tends to infinity:%
\[
\tau_{m}\frac{\partial a\left(  x,t\right)  }{\partial t}=-a\left(
x,t\right)  +%
%TCIMACRO{\dint \nolimits_{-\infty}^{\infty}}%
%BeginExpansion
{\displaystyle\int\nolimits_{-\infty}^{\infty}}
%EndExpansion
w(x,y)f\left(  a\left(  y,t\right)  \right)  dy,
\]%
\[
\tau_{s}\frac{\partial v\left(  x,t\right)  }{\partial t}=-v\left(
x,t\right)  +f\left(
%TCIMACRO{\dint \nolimits_{-\infty}^{\infty}}%
%BeginExpansion
{\displaystyle\int\nolimits_{-\infty}^{\infty}}
%EndExpansion
w(x,y)v\left(  y,t\right)  dy\right)  ,
\]
where $x\in\mathbb{R}$, $t\in\mathbb{R}_{+}$. In paricular, this approach
discard any hierarchical organization of the neurons.

The Wilson-Cowan model describes the evolution of excitatory and inhibitory
activity in a synaptically coupled neuronal network. The model is given by the
following system of non-linear integro-differential evolution equations:%

\begin{equation}
\left\{
\begin{array}
[c]{l}%
\tau\frac{\partial E(x,t)}{\partial t}=-E(x,t)+\\
\\
\left(  1-r_{E}E(x,t)\right)  f_{E}\left(  w_{EE}\left(  x\right)  \ast
E(x,t)-w_{EI}\left(  x\right)  \ast I(x,t)+h_{E}\left(  x,t\right)  \right) \\
\\
\tau\frac{\partial I(x,t)}{\partial t}=-I(x,t)+\\
\\
\left(  1-r_{I}I(x,t)\right)  f_{I}\left(  w_{IE}\left(  x\right)  \ast
E(x,t)-w_{II}(x)\ast I(x,t)+h_{I}\left(  x,t\right)  \right)  ,
\end{array}
\right.  \label{Model_3}%
\end{equation}
where $E(x,t)$ is a temporal coarse-grained variable describing the proportion
of excitatory neuron firing per unit of time at position $x\in\mathbb{R}$ at
the \ instant $t\in\mathbb{R}_{+}$. Similarly the variable $I(x,t)$ represents
the activity of the inhibitory population of neurons. The `$\ast$' denotes the
spatial convolution. The main parameters of the model are the strength of the
connections between each subtype of population ($w_{EE}$, $w_{IE}$, $w_{EI}$,
$w_{II}$) and the strength of input to each subpopulation ($h_{E}\left(
x,t\right)  $, $h_{I}\left(  x,t\right)  $). This model generates a diversity
of dynamical behaviors that are representative of observed activity in the
brain, like multistability, oscillations, traveling waves, and spatial
patterns, see, e.g., \cite{Neural-Fields} and the references therein.

\subsection{\label{Section_Neuwwons_geometry}Neurons geometry and small-world
networks}

Nowadays, there are extensive databases of neuronal wiring diagrams
(connection matrices) of the invertebrates' and mammals' cerebral cortexes.
The connection matrices are adjacency matrices of weighted directed graphs,
where the vertices represent neuron populations or regions in a cortex. These
matrices correspond to the kernels $w_{AB}$, $A$, $B\in\left\{  E,I\right\}
$, in model (\ref{Model_3}), or to the kernels $w_{ij}$ in models
(\ref{Model_1})-(\ref{Model_2}). Based on these data, several researchers
hypothesized that cortical neural networks are arranged in fractal or
self-similar patterns and have the small-world property, see, e.g.,
\cite{Sporns et al 1}-\cite{Skoch et al}, and the references therein. It is
widely accepted that the brain is a small-world network, see, e.g.,
\cite{Sporns}-\cite{Muldoon et al}, and the references therein. The
small-worldness is believed to be a crucial aspect of efficient brain
organization that confers significant advantages in signal processing,
furthermore, the small-world organization is deemed essential for healthy
brain function, see, e.g., \cite{Hilgetag et al},\ and the references
therein.\ A small-world network has a topology that produces short paths
across the whole network, i.e., given two nodes, there is a short path between
them (the six degrees of separation phenomenon). In turn, this implies the
existence of long-range interactions in the network.

The compatibility of the Wilson-Cowan\ type models with the small-world
network property requires a non-negligible interaction between any two groups
of neurons, i.e., $w_{AB}\left(  x\right)  >\varepsilon>0$, for any $x$, and
for$\ A$, $B\in\left\{  E,I\right\}  $, where the constant $\varepsilon>0$ is
independent of $x$. Since the kernels $w_{AB}$ should be integrable, such as
we pointed out in the simplest model (\ref{Model_2}), then, it is necessary to
replace $\left(  \mathbb{R},+\right)  $ by a compact subgroup, however,
$\left(  \mathbb{R},+\right)  $ does not have non-trivial subgroups, and
consequently, the Wilson-Cowan models on $\left(  \mathbb{R},+\right)  $ are
not compatible with small-world network property.

\subsection{\label{Section_CNN}Cellular neural Networks and complex
multi-level systems}

In the late 80s Chua and Yang introduced a new natural computing paradigm
called the cellular neural networks CNN which includes the cellular automata
as a particular case. This paradigm has been extremely successful in various
applications in vision, robotics and remote sensing, etc. See \cite{Chua-Yang}%
-\cite{Slavova} and the references therein.

Let $(\mathcal{M},\left\vert \cdot\right\vert )$ be a normed $\mathbb{R}%
$-vector space, where $\mathcal{M}$ is a finite set. An element $i$ of
$\mathcal{M}$ is called a cell. A discrete CNN is a dynamical system
CNN$(\mathbb{A},\mathbb{B},U,Z)$ on $\mathcal{M}$. The state $X(i,t)\in
\mathbb{R}$ of cell $i$ is described by the following system of differential equations:

(i) state equation:
\[
\frac{dX(i,t)}{dt}=-X(i,t)+\sum_{j\in\mathcal{M}}\mathbb{A}(i,j)Y(j,t)+\sum
_{j\in\mathcal{M}}\mathbb{B}(i,j)U(j)+Z(i)\text{, }i\in\mathcal{M}\text{,\ }%
\]

(ii) output equation:
\[
Y(i,t)=f(X(i,t))\text{, }i\in\mathcal{M}\text{,}%
\]
where $Y(i,t)\in\mathbb{R}$ is the output of cell $i$ at the time $t$,
$f:\mathbb{R}\rightarrow\mathbb{R}$ is a bounded Lipschitz function satisfying
$f(0)=0$. The function $U(i)\in\mathbb{R}$ is the input of the cell $i$,
$Z(i)\in\mathbb{R}$ is the threshold of cell $i$, and $\mathbb{A}%
,\mathbb{B}:\mathcal{M}\times\mathcal{M}\rightarrow\mathbb{R}$ are the
feedback operator and feedforward operator, respectively.

Not all the cells of $\mathcal{M}$ are active. A cell $i$ is connected with
cell $j$ if $\mathbb{A}(i,j)\neq0$\ or $\mathbb{B}(i,j)\neq0$ for some
$j\in\mathcal{M}$. Then, a $p$-adic discrete CNN is a dynamical system on
\[
C_{\mathcal{M}}:=\left\{  i\in\mathcal{M};\mathbb{A}(i,j)\neq0\ \text{or
}\mathbb{B}(i,j)\neq0\text{ for some }j\in\mathcal{M}\right\}  .
\]
The topology of a\ $p$-adic discrete $\text{CNN depends on the func\-tions
}\mathbb{A}$, $\mathbb{B}:$ $\mathcal{M}\times\mathcal{M}\rightarrow
\mathbb{R}$. The discrete CNNs\ satisfying%
\begin{equation}
\mathbb{A}(i,j)=\mathbb{A}(\left\vert i-j\right\vert )\text{, }\mathbb{B}%
(i,j)=\mathbb{B}(\left\vert i-j\right\vert ), \label{Hyp_1}%
\end{equation}
which are discrete CNNs having the space-invariant property.

The CNNs are bioinspired on the Wilson-Cowan models; see (\ref{Model_1}). Most
of the literature about CNNs is dedicated to the study and applications of NNs
whose topology comes from a finite lattice contained in some $\mathbb{R}^{n}$.
We already pointed out that it is widely accepted that hierarchical
organization plays a central role in biological NNs. Since most artificial NNs
are bioinspired, it is natural to conclude that hierarchical artificial NNs
are relevant computational paradigms. We propose to study hierarchical NNs as
complex multi-level systems. These systems are made up of several subsystems
and are characterized by emergent behavior resulting from nonlinear
interactions between subsystems for multiple levels of organization; see,
e.g., \cite{Iordache}-\cite{KKZuniga}, and the references therein.

An\ non-Archimedean vector space $\left(  M,\left\Vert \cdot\right\Vert
\right)  $ is a normed vector space whose norm satisfies
\[
\left\Vert x+y\right\Vert \leq\max\left\{  \left\Vert x\right\Vert ,\left\Vert
y\right\Vert \right\}  ,
\]
for any two vectors $x$, $y$ in $M$.\ In such a space, the balls are organized
in a hierarchical form. This type of space plays a central role in formulating
models of complex multi-level systems. The field of $p$-adic numbers
$\mathbb{Q}_{p}$ and the field of formal Laurent series $\mathbb{F}_{p}\left(
\left(  T\right)  \right)  $ are paramount examples of non-Archimedean vector
spaces. Methods of non-Archimedean analysis have been successfully used to
construct models for hierarchical NNs; see, e.g., \cite{Khrenikov2}%
-\cite{Zuniga-Nonlinearity}, and the references therein.

\subsection{$p$-Adic numbers and tree-like structures}

This section reviews some basic results on $p$-adic analysis required in this
article. For a detailed exposition on $p$-adic analysis, the reader may
consult \cite{Alberio et al}-\cite{Taibleson}.

\subsubsection{$p$-Adic integers}

From now on, $p$ denotes a fixed prime number. The ring of $p-$adic integers
$\mathbb{Z}_{p}$ is defined as the completion of the ring of integers
$\mathbb{Z}$ with respect to the $p-$adic norm $|\cdot|_{p}$, which is defined
as
\begin{equation}
|x|_{p}=%
\begin{cases}
0 & \text{if }x=0\\
p^{-\gamma} & \text{if }x=p^{\gamma}a\in\mathbb{Z},
\end{cases}
\label{p-norm}%
\end{equation}
where $a$ is an integers coprime with $p$. The integer $\gamma=ord_{p}%
(x):=ord(x)$, with $ord(0):=+\infty$, is called the\textit{\ }$p-$\textit{adic
order of} $x$.

Any non-zero $p-$adic integer $x$ has a unique expansion of the form%
\[
x=x_{k}p^{k}+x_{k+1}p^{k+1}+\ldots,\text{ }%
\]
with $x_{k}\neq0$, where $k$ is a non-negative integer, and the $x_{j}$s \ are
numbers from the set $\left\{  0,1,\ldots,p-1\right\}  $. There are natural
field operations, sum, and multiplication, on $p$-adic integers. The norm
$|\cdot|_{p}$, see (\ref{p-norm}), extends to $\mathbb{Z}_{p}$ as $\left\vert
x\right\vert _{p}=p^{-k}$ for a nonzero $p$-adic integer $x$.

The metric space $\left(  \mathbb{Z}_{p},\left\vert \cdot\right\vert
_{p}\right)  $ is a complete ultrametric space. Ultrametric means that
$\left\vert x+y\right\vert _{p}\leq\max\left\{  \left\vert x\right\vert
_{p},\left\vert y\right\vert _{p}\right\}  $. As a topological space
$\mathbb{Z}_{p}$\ is homeomorphic to a Cantor-like subset of the real line,
see Figure \ref{Figure 1}, and the references \cite{Alberio et al}%
-\cite{V-V-Z}, \cite{Chistyakov}.

For $r\in\mathbb{N}$, denote by $B_{-r}(a)=\{x\in\mathbb{Z}_{p};\left\vert
x-a\right\vert _{p}\leq p^{-r}\}$ \textit{the ball of radius }$p^{-r}$
\textit{with center at} $a\in\mathbb{Z}_{p}$, and take $B_{-r}(0):=B_{-r}$.
{The ball $B_{0}$ equals \textit{the ring of }$p-$\textit{adic integers
}$\mathbb{Z}_{p}$ (the unit ball).} We use $\Omega\left(  p^{r}\left\vert
x-a\right\vert _{p}\right)  $ to denote the characteristic function of the
ball $B_{-r}(a)$. Two balls in $\mathbb{Z}_{p}$ are either disjoint or one is
contained in the other. The balls are compact subsets, thus $\left(
\mathbb{Z}_{p},\left\vert \cdot\right\vert _{p}\right)  $ is a compact
topological space.%

%TCIMACRO{\FRAME{ftbpFU}{4.7478in}{2.5754in}{0pt}{\Qcb{Based upon
%\cite{Chistyakov}, we construct an embedding $\mathfrak{f}:\mathbb{Z}%
%_{p}\rightarrow\mathbb{R}^{2}$. The figure shows the images of $\mathfrak{f}%
%(\mathbb{Z}_{2})$ and $\mathfrak{f}(\mathbb{Z}_{3})$. This computation
%requires a truncation of the $p$-adic integers. We use $\mathbb{Z}_{2}%
%/2^{14}\mathbb{Z}_{2}$ and $\mathbb{Z}_{3}/3^{10}\mathbb{Z}_{3}$,
%respectively.}}{\Qlb{Figure 1}}{Figure 1}%
%{\special{ language "Scientific Word";  type "GRAPHIC";  display "USEDEF";
%valid_file "T";  width 4.7478in;  height 2.5754in;  depth 0pt;
%original-width 7.139in;  original-height 3.2984in;  cropleft "0";
%tempfilename 'SEDD7S05.wmf';tempfile-properties "XPR";}} }%
%BeginExpansion
\begin{figure}
[h]
\begin{center}
\includegraphics[width=0.75\textwidth]
{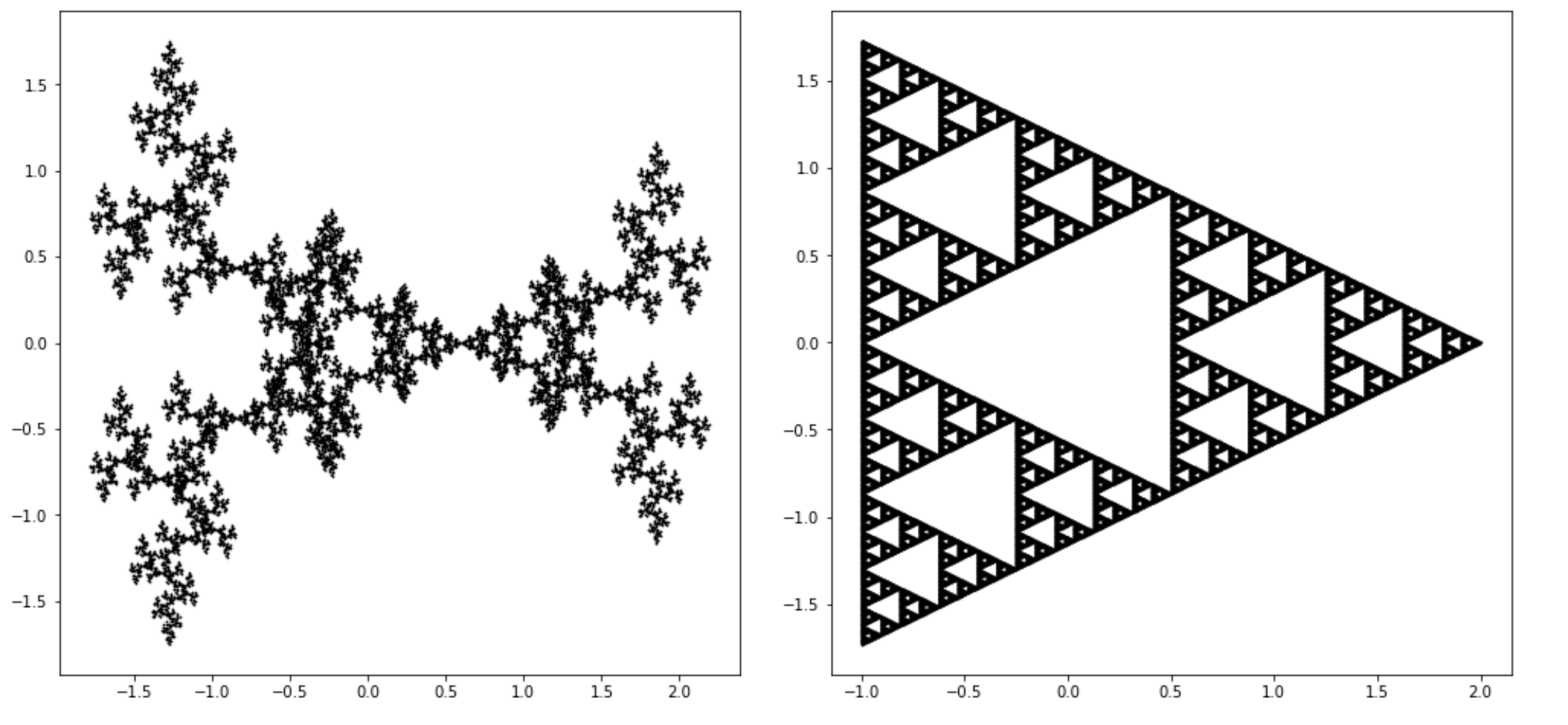}%
\caption{Based upon \cite{Chistyakov}, we construct an embedding
$\mathfrak{f}:\mathbb{Z}_{p}\rightarrow\mathbb{R}^{2}$. The figure shows the
images of $\mathfrak{f}(\mathbb{Z}_{2})$ and $\mathfrak{f}(\mathbb{Z}_{3})$.
This computation requires a truncation of the $p$-adic integers. We use
$\mathbb{Z}_{2}/2^{14}\mathbb{Z}_{2}$ and $\mathbb{Z}_{3}/3^{10}\mathbb{Z}%
_{3}$, respectively.}%
\label{Figure 1}%
\end{center}
\end{figure}
%EndExpansion

\subsubsection{Tree-like structures}

The set of $p$-adic integers modulo $p^{l}$, $l\geq1$, consists of all the
integers of the form $i=i_{0}+i_{1}p+\ldots+i_{l-1}p^{l-1}$. These numbers
form a complete set of representatives for the elements of the additive group
$G_{l}=\mathbb{Z}_{p}/p^{l}\mathbb{Z}_{p}$, which is isomorphic to the set of
integers $\mathbb{Z}/p^{l}\mathbb{Z}$ (written in base $p$) modulo $p^{l}$. By
restricting $\left\vert \cdot\right\vert _{p}$ to $G_{l}$, it becomes a normed
space, and $\left\vert G_{l}\right\vert _{p}=\left\{  0,p^{-\left(
l-1\right)  },\cdots,p^{-1},1\right\}  $. With the metric induced by
$\left\vert \cdot\right\vert _{p}$, $G_{l}$ becomes a finite ultrametric
space. In addition, $G_{l}$ can be identified with the set of branches
(vertices at the top level) of a rooted tree with $l$ levels and $p^{l}$
branches. By definition, the tree's root is the only vertex at level $0$.
There are exactly $p$ vertices at level $1$, which correspond with the
possible values of the digit $i_{0}$ in the $p$-adic expansion of $i$. Each of
these vertices is connected to the root by a non-directed edge. At level $k$,
with $2\leq k\leq l-1$, there are exactly $p^{k}$ vertices, \ each vertex
corresponds to a truncated expansion of $i$ of the form $i_{0}+\cdots
+i_{k-1}p^{k-1}$. The vertex corresponding to $i_{0}+\cdots+i_{k-1}p^{k-1}$ is
connected to a vertex $i_{0}^{\prime}+\cdots+i_{k-2}^{\prime}p^{k-2}$ at the
level $k-1$ if and only if $\left(  i_{0}+\cdots+i_{k-1}p^{k-1}\right)
-\left(  i_{0}^{\prime}+\cdots+i_{k-2}^{\prime}p^{k-2}\right)  $ is divisible
by $p^{k-1}$. See Figure \ref{Figure 1}. The balls $B_{-r}(a)=a+p^{r}%
\mathbb{Z}_{p}$ are infinite rooted trees.%

%TCIMACRO{\FRAME{ftbpFU}{4.4477in}{2.5131in}{0pt}{\Qcb{The rooted tree
%associated with the group $\mathbb{Z}_{2}/2^{3}\mathbb{Z}_{2}$. The elements
%of $\mathbb{Z}_{2}/2^{3}\mathbb{Z}_{2}$ have the form $i=i_{0}+i_{1}%
%2+i_{2}2^{2}$,$\;i_{0}$, $i_{1}$, $i_{2}\in\{0,1\}$. The distance satisfies
%$-\log_{2}\left\vert i-j\right\vert _{2}=$level of the first common ancestor
%of $i$, $j$. Figure taken from \cite{Zuniga et al}}}{\Qlb{Figure 2}}{Figure
%2}{\special{ language "Scientific Word";  type "GRAPHIC";
%maintain-aspect-ratio TRUE;  display "USEDEF";  valid_file "T";
%width 4.4477in;  height 2.5131in;  depth 0pt;  original-width 8.8885in;
%original-height 5.0004in;  cropleft "0";  croptop "1";  cropright "1";
%cropbottom "0";  tempfilename 'SEDD2I04.wmf';tempfile-properties "XPR";}} }%
%BeginExpansion
\begin{figure}
[h]
\begin{center}
\includegraphics[width=0.75\textwidth]%
{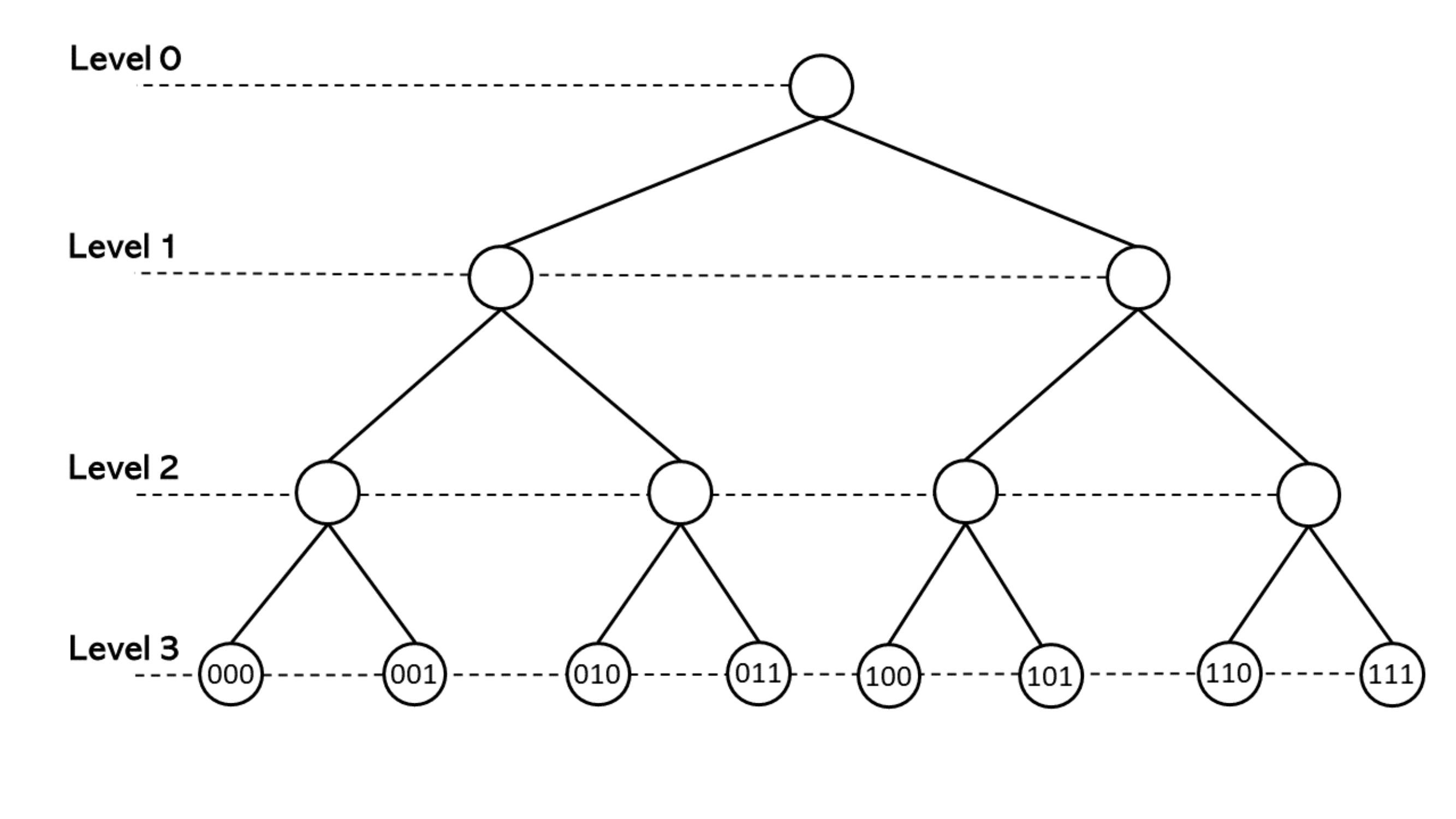}%
\caption{The rooted tree associated with the group $\mathbb{Z}_{2}%
/2^{3}\mathbb{Z}_{2}$. The elements of $\mathbb{Z}_{2}/2^{3}\mathbb{Z}_{2}$
have the form $i=i_{0}+i_{1}2+i_{2}2^{2}$,$\;i_{0}$, $i_{1}$, $i_{2}%
\in\{0,1\}$. The distance satisfies $-\log_{2}\left\vert i-j\right\vert _{2}%
=$level of the first common ancestor of $i$, $j$. Figure taken from
\cite{Zuniga et al}}%
\label{Figure 2}%
\end{center}
\end{figure}
%EndExpansion

\subsubsection{The Bruhat-Schwartz space in the unit ball}

A real-valued function $\varphi$ defined on $\mathbb{Z}_{p}$ is \textit{called
Bruhat-Schwartz function (or a test function)} if for any $x\in\mathbb{Z}_{p}$
there exist an integer $l\in\mathbb{N=}\left\{  0,1,\ldots\right\}  $ such
that%
\begin{equation}
\varphi(x+x^{\prime})=\varphi(x)\text{ for any }x^{\prime}\in B_{l}.
\label{local_constancy}%
\end{equation}
The $\mathbb{R}$-vector space of Bruhat-Schwartz functions supported in the
unit ball is denoted by $\mathcal{D}(\mathbb{Z}_{p})$. For $\varphi
\in\mathcal{D}(\mathbb{Z}_{p})$, the largest number $l=l(\varphi)$ satisfying
(\ref{local_constancy}) is called \textit{the exponent of local constancy (or
the parameter of constancy) of} $\varphi$. A test function with exponent of
local constancy $l$ has the form%
\[
\varphi\left(  x\right)  =%
%TCIMACRO{\tsum \limits_{i\in G_{l}}}%
%BeginExpansion
{\textstyle\sum\limits_{i\in G_{l}}}
%EndExpansion
\varphi\left(  i\right)  \Omega\left(  p^{l}\left\vert x-i\right\vert
_{p}\right)  \text{, \ }\varphi\left(  i\right)  \in\mathbb{R}\text{,}%
\]
where $i=i_{0}+i_{1}p+\ldots+i_{l-1}p^{l-1}\in G_{l}=\mathbb{Z}_{p}%
/p^{l}\mathbb{Z}_{p}$. Notice that if $l=0$, $G_{0}=\left\{  0\right\}  $.
These functions form a finite dimensional vector space $\mathcal{D}%
^{l}(\mathbb{Z}_{p})$ spanned by the basis $\left\{  \Omega\left(
p^{l}\left\vert x-i\right\vert _{p}\right)  \right\}  _{i\in G_{l}}$. By
identifying $\varphi\in\mathcal{D}^{l}(\mathbb{Z}_{p})$ with the column vector
$\left[  \varphi\left(  i\right)  \right]  _{i\in G_{l}}\in\mathbb{R}%
^{\#G_{l}}$, we get that $\mathcal{D}^{l}(\mathbb{Z}_{p})$ is isomorphic (as
Banach space) to $\mathbb{R}^{\#G_{l}}$ endowed with the norm%
\[
\left\Vert \left[  \varphi\left(  i\right)  \right]  _{i\in G_{l}^{N}%
}\right\Vert =\max_{i\in G_{l}}\left\vert \varphi\left(  i\right)  \right\vert
.
\]
Furthermore,
\[
\mathcal{D}^{l}\hookrightarrow\mathcal{D}^{l+1}\hookrightarrow\mathcal{D}%
(\mathbb{Z}_{p}),
\]
where $\hookrightarrow$ denotes a continuous embedding, and $\mathcal{D}%
(\mathbb{Z}_{p})=\cup_{l\in\mathbb{N}}\mathcal{D}^{l}(\mathbb{Z}_{p})$.

\subsubsection{The Haar measure}

Since $(\mathbb{Z}_{p},+)$ is a compact topological group, there exists a Haar
measure $dx$, which is invariant under translations, i.e., $d(x+a)=dx$,
\cite{Halmos}. If we normalize this measure by the condition $\int
_{\mathbb{Z}_{p}}dx=1$, then $dx$ is unique. It follows immediately that
\[%
%TCIMACRO{\tint \limits_{B_{-r}(a)}}%
%BeginExpansion
{\textstyle\int\limits_{B_{-r}(a)}}
%EndExpansion
dx=%
%TCIMACRO{\tint \limits_{a+p^{r}\mathbb{Z}_{p}}}%
%BeginExpansion
{\textstyle\int\limits_{a+p^{r}\mathbb{Z}_{p}}}
%EndExpansion
dx=p^{-r}%
%TCIMACRO{\tint \limits_{\mathbb{Z}_{p}}}%
%BeginExpansion
{\textstyle\int\limits_{\mathbb{Z}_{p}}}
%EndExpansion
dy=p^{-r}\text{, }r\in\mathbb{N}\text{.}%
\]
In a few occasions, we use the two-dimensional Haar measure $dxdy$ of the
additive group $(\mathbb{Z}_{p}\times\mathbb{Z}_{p},+)$ normalize this measure
by the condition $\int_{\mathbb{Z}_{p}}\int_{\mathbb{Z}_{p}}dxdy=1$.

\subsubsection{Other function spaces}

The space $\mathcal{D}(\mathbb{Z}_{p})$ is dense in
\[
L^{\rho}\left(  \mathbb{Z}_{p}\right)  =L^{\rho}=\left\{  \varphi
:\mathbb{Z}_{p}\rightarrow\mathbb{R};\left(
%TCIMACRO{\dint \limits_{\mathbb{Z}_{p}}}%
%BeginExpansion
{\displaystyle\int\limits_{\mathbb{Z}_{p}}}
%EndExpansion
\left\vert \varphi\left(  x\right)  \right\vert ^{\rho}dx\right)  ^{\frac
{1}{\rho}}<\infty\right\}  ,
\]
for $1\leq\rho<\infty$, see e.g. \cite[Section 4.3]{Alberio et al}.

We denote by $C(\mathbb{Z}_{p})$ the $\mathbb{R}$-vector space of continuous
functions defined on $\mathbb{Z}_{p}$.$\ $The space of test functions
$\mathcal{D}(\mathbb{Z}_{p})$ is dense in $C(\mathbb{Z}_{p})$ with respect to
the norm $\left\Vert \phi\right\Vert _{\infty}=\sup_{x\in\mathbb{Z}_{p}}%
|\phi(x)|$. When formulating solutions of initial value problems, the notation
$C(\mathbb{Z}_{p},\mathbb{R})$ will be also used.

\section{$p$-Adic continuous CNNs}

In this section, we quickly review some of the results of
\cite{Zambrano-Zuniga-1}. Initially, these results were formulated in
$\mathbb{Q}_{p}^{N}$, but here we reformulate them in $\mathbb{Z}_{p}$. The
adaption is straightforward.

We say that a \ function $f:\mathbb{R}\rightarrow\mathbb{R}$ is called a
Lipschitz function if there exists a real constant $L(f)>0$ such that, for all
$x,y\in\mathbb{R}$, $|f(x)-f(y)|\leq L(f)|x-y|$. We assume $f(0)=0$. A
relevant example is
\[
f(x)=\frac{1}{2}\left(  |x+1|-|x-1|\right)  .
\]

We define $\mathcal{X}_{\infty}(\mathbb{Z}_{p}):=\mathcal{X}_{\infty}=\left(
C(\mathbb{Z}_{p}),\left\Vert \cdot\right\Vert _{\infty}\right)  $, where
$\left\Vert \phi\right\Vert _{\infty}=\sup_{x\in\mathbb{Z}_{p}}|\phi(x)|$, and
$\mathcal{X}_{M}:=\left(  \mathcal{D}^{M}(\mathbb{Z}_{p}),\left\Vert
\cdot\right\Vert _{\infty}\right)  $ for $M\geq1$. The spaces $\mathcal{X}%
_{\infty}$ and $\mathcal{X}_{M}$ are Banach spaces.

Given $A(x,y)$, $B(x,y)\in L^{1}(\mathbb{Z}_{p}\times\mathbb{Z}_{p})$, and
$U$, $Z\in\mathcal{X}_{\infty}$, a $p$-adic continuous CNN, denoted as
CNN$(A,B,U,Z)$, is the dynamical system given by the following
integrodifferential equations: (i) state equation:%
\begin{equation}
\frac{\partial X(x,t)}{\partial t}=-X(x,t)+%
%TCIMACRO{\dint \limits_{\mathbb{Z}_{p}}}%
%BeginExpansion
{\displaystyle\int\limits_{\mathbb{Z}_{p}}}
%EndExpansion
A(x,y)Y(y,t)dy+%
%TCIMACRO{\dint \limits_{\mathbb{Z}_{p}}}%
%BeginExpansion
{\displaystyle\int\limits_{\mathbb{Z}_{p}}}
%EndExpansion
B(x,y)U(y)dy+Z(x), \label{Continuous_CNN}%
\end{equation}
where $x\in\mathbb{Z}_{p}$, $t\geq0$, and (ii) output equation:
$Y(x,t)=f(X(x,t))$. We say that $X(x,t)\in\mathbb{R}$ is the \textit{state of
cell} $x$ at the time $t$, $Y(x,t)\in\mathbb{R}$ is \textit{the output of
cell} $x$ at the time $t$. Function $A(x,y)$ is \textit{the kernel of the
feedback operator, while \ function } $B(x,y)$ is t\textit{he kernel of the
feedforward operator}. Function $U$ is \textit{the input of the $\text{CNN}$},
while function $Z$ is \textit{the threshold of the $\text{CNN.}$}

We focus mainly in continuous CNNs having the space invariant property, i.e.
$A(x,y)=A(\left\vert x-y\right\vert _{p})$ and $B(x,y)=B(\left\vert
x-y\right\vert _{p})$ for some $A,B\in L^{1}$, however our results are valid
for general $p$-adic continuous $\text{CNN}$s.

\subsection{Discretization of $p$-adic continuous CNNs}

The $p$-adic continuous CNNs are mathematical models of hierarchical CNNs. For
practical pourposes, discrete versions of them are required. By the H\"{o}lder
inequality,%
\[
\mathcal{D}(\mathbb{Z}_{p})\subset C(\mathbb{Z}_{p})\subset L^{\rho}\left(
\mathbb{Z}_{p}\right)  \subset L^{1}\left(  \mathbb{Z}_{p}\right)  \text{, for
}\rho\in\left(  1,\infty\right]  ,
\]
with $\mathcal{D}(\mathbb{Z}_{p})$ dense in $L^{1}\left(  \mathbb{Z}%
_{p}\right)  $. Then, any function $f\in L^{\rho}\left(  \mathbb{Z}%
_{p}\right)  $, $\rho\in\left[  1,\infty\right]  ,$ can be approximated in the
norm $\left\Vert \cdot\right\Vert _{\rho}$ by a function from $\mathcal{D}%
^{l}(\mathbb{Z}_{p})\subset\mathcal{D}(\mathbb{Z}_{p})$, for some $l\geq1$.
For this reason, a discretization of a $p$-adic continuous $\text{CNN}%
(A,B,U,Z)$ is obtained assuming that $X(\cdot,t)$, $A$, $Y(\cdot,t)$, $B$, $U$
and $Z$ belong to $\mathcal{D}^{l}(\mathbb{Z}_{p})$, i.e.%
\begin{gather*}
X(x,t)=%
%TCIMACRO{\dsum \limits_{i\in G_{l}}}%
%BeginExpansion
{\displaystyle\sum\limits_{i\in G_{l}}}
%EndExpansion
X(i,t)\Omega\left(  p^{l}\left\vert x-i\right\vert _{p}\right)  \text{,
\ }Y(x,t)=%
%TCIMACRO{\dsum \limits_{i\in G_{l}}}%
%BeginExpansion
{\displaystyle\sum\limits_{i\in G_{l}}}
%EndExpansion
Y(i,t)\Omega\left(  p^{l}\left\vert x-i\right\vert _{p}\right)  ,\\
U(x)=%
%TCIMACRO{\dsum \limits_{i\in G_{l}}}%
%BeginExpansion
{\displaystyle\sum\limits_{i\in G_{l}}}
%EndExpansion
U(i)\Omega\left(  p^{l}\left\vert x-i\right\vert _{p}\right)  \text{, \ }Z(x)=%
%TCIMACRO{\dsum \limits_{i\in G_{l}}}%
%BeginExpansion
{\displaystyle\sum\limits_{i\in G_{l}}}
%EndExpansion
Z(i)\Omega\left(  p^{l}\left\vert x-i\right\vert _{p}\right)  ,\\
A(x,y)=%
%TCIMACRO{\dsum \limits_{i\in G_{l}}}%
%BeginExpansion
{\displaystyle\sum\limits_{i\in G_{l}}}
%EndExpansion
\text{ }%
%TCIMACRO{\dsum \limits_{j\in G_{l}}}%
%BeginExpansion
{\displaystyle\sum\limits_{j\in G_{l}}}
%EndExpansion
A(i,j)\Omega\left(  p^{l}\left\vert x-i\right\vert _{p}\right)  \Omega\left(
p^{l}\left\vert y-j\right\vert _{p}\right)  ,\\
B(x,y)=%
%TCIMACRO{\dsum \limits_{i\in G_{l}}}%
%BeginExpansion
{\displaystyle\sum\limits_{i\in G_{l}}}
%EndExpansion
\text{ }%
%TCIMACRO{\dsum \limits_{j\in G_{l}}}%
%BeginExpansion
{\displaystyle\sum\limits_{j\in G_{l}}}
%EndExpansion
B(i,j)\Omega\left(  p^{l}\left\vert x-i\right\vert _{p}\right)  \Omega\left(
p^{l}\left\vert y-j\right\vert _{p}\right)  .
\end{gather*}
On the other hand, if $f:\mathbb{R}\rightarrow\mathbb{R}$, then
\[
f\left(  X(x,t)\right)  =%
%TCIMACRO{\dsum \limits_{i\in G_{l}}}%
%BeginExpansion
{\displaystyle\sum\limits_{i\in G_{l}}}
%EndExpansion
f(X(i,t))\Omega\left(  p^{l}\left\vert x-i\right\vert _{p}\right)
=Y(x,t)\text{.}%
\]
Now,%
\[
\frac{\partial}{\partial t}X(x,t)=%
%TCIMACRO{\dsum \limits_{i\in G_{l}}}%
%BeginExpansion
{\displaystyle\sum\limits_{i\in G_{l}}}
%EndExpansion
\frac{\partial}{\partial t}X(i,t)\Omega\left(  p^{l}\left\vert x-i\right\vert
_{p}\right)  ,
\]
and%
\begin{align*}
&
%TCIMACRO{\dint \limits_{\mathbb{Z}_{p}}}%
%BeginExpansion
{\displaystyle\int\limits_{\mathbb{Z}_{p}}}
%EndExpansion
A(x,y)f\left(  X(y,t)\right)  dy\\
&  =%
%TCIMACRO{\dsum \limits_{i\in G_{l}}}%
%BeginExpansion
{\displaystyle\sum\limits_{i\in G_{l}}}
%EndExpansion
\left\{
%TCIMACRO{\dsum \limits_{j\in G_{l}}}%
%BeginExpansion
{\displaystyle\sum\limits_{j\in G_{l}}}
%EndExpansion
A(i,j)f(X(j,t))%
%TCIMACRO{\dint \limits_{\mathbb{Z}_{p}}}%
%BeginExpansion
{\displaystyle\int\limits_{\mathbb{Z}_{p}}}
%EndExpansion
\Omega\left(  p^{l}\left\vert y-j\right\vert _{p}\right)  dy\right\}
\Omega\left(  p^{l}\left\vert x-i\right\vert _{p}\right) \\
&  =p^{-l}%
%TCIMACRO{\dsum \limits_{i\in G_{l}}}%
%BeginExpansion
{\displaystyle\sum\limits_{i\in G_{l}}}
%EndExpansion
\left\{
%TCIMACRO{\dsum \limits_{j\in G_{l}}}%
%BeginExpansion
{\displaystyle\sum\limits_{j\in G_{l}}}
%EndExpansion
A(i,j)Y(j,t))\right\}  \Omega\left(  p^{l}\left\vert x-i\right\vert
_{p}\right)  .
\end{align*}
Similarly,%
\[%
%TCIMACRO{\dint \limits_{\mathbb{Z}_{p}}}%
%BeginExpansion
{\displaystyle\int\limits_{\mathbb{Z}_{p}}}
%EndExpansion
B(x,y)U(y)dy=p^{-l}%
%TCIMACRO{\dsum \limits_{i\in G_{l}}}%
%BeginExpansion
{\displaystyle\sum\limits_{i\in G_{l}}}
%EndExpansion
\left\{
%TCIMACRO{\dsum \limits_{j\in G_{l}}}%
%BeginExpansion
{\displaystyle\sum\limits_{j\in G_{l}}}
%EndExpansion
B(i,j)U(j))\right\}  \Omega\left(  p^{l}\left\vert x-i\right\vert _{p}\right)
.
\]
Therefore,%
\[
\frac{\partial}{\partial t}X(i,t)=-X(i,t)+%
%TCIMACRO{\dsum \limits_{j\in G_{l}}}%
%BeginExpansion
{\displaystyle\sum\limits_{j\in G_{l}}}
%EndExpansion
p^{-l}A(i,j)Y(j,t)+%
%TCIMACRO{\dsum \limits_{j\in G_{l}}}%
%BeginExpansion
{\displaystyle\sum\limits_{j\in G_{l}}}
%EndExpansion
p^{-l}B(i,j)U(j)+Z(i)\text{, }%
\]
for $i\in G_{l}$, and $Y(i,t)=f(X(i,t))$, for $i\in G_{l}$. This is exactly a
$p$-adic discrete CNN with $\mathcal{M}=G_{l}$, $\left\vert x-y\right\vert
=\left\vert x-y\right\vert _{p}$, $\mathbb{A}(i,j)=p^{-l}A(i,j)$,
$\mathbb{B}(i,j)=p^{-l}B(i,j)$. Notice that a $p$-adic discrete CNN is a
dynamical system on
\[
C_{l}:=\left\{  i\in G_{l};\mathbb{A}(i,j)\neq0\ \text{or }\mathbb{B}%
(i,j)\neq0\text{ for some }j\in G_{l}\right\}  ;
\]
see Figure \ref{Figure 3}. Intuitively a $p$-adic continuous CNN has
infinitely many layers, each layer corresponds to some $l$, which are
organized in a hierarchical structure. For practical purposes, a $p$-adic
continuous CNN is realized as a $p$-adic discrete CNN for $l$ sufficiently large.%

%TCIMACRO{\FRAME{ftbpFU}{4.817in}{2.0384in}{0pt}{\Qcb{A discrete $2$-adic CNN
%with $8$ cells: $C_{3}=\{0,1,2,3,4,5,7\}\subset\mathbb{Z}_{2}/2^{3}%
%\mathbb{Z}_{2}\subset2^{-3}\mathbb{Z}_{2}/2^{3}\mathbb{Z}_{2}$. We set
%$\mathbb{B}=0$ and $\mathbb{A}(i,j)=\left[  a_{i,j}\right]  $, with
%$a_{i,j}\neq0$ if $|i-j|_{2}=1/2$ and $i$, $j\in C_{3}$; $a_{i,j}=0$
%otherwise. Taken from \cite{Zambrano-Zuniga-1}}}{\Qlb{Figure 3}}{Figure
%3}{\special{ language "Scientific Word";  type "GRAPHIC";  display "USEDEF";
%valid_file "T";  width 4.817in;  height 2.0384in;  depth 0pt;
%original-width 5.8954in;  original-height 2.0003in;  cropleft "0";
%tempfilename 'SEEQSW00.wmf';tempfile-properties "XPR";}} }%
%BeginExpansion
\begin{figure}
[h]
\begin{center}
\includegraphics[width=0.85\textwidth]%
{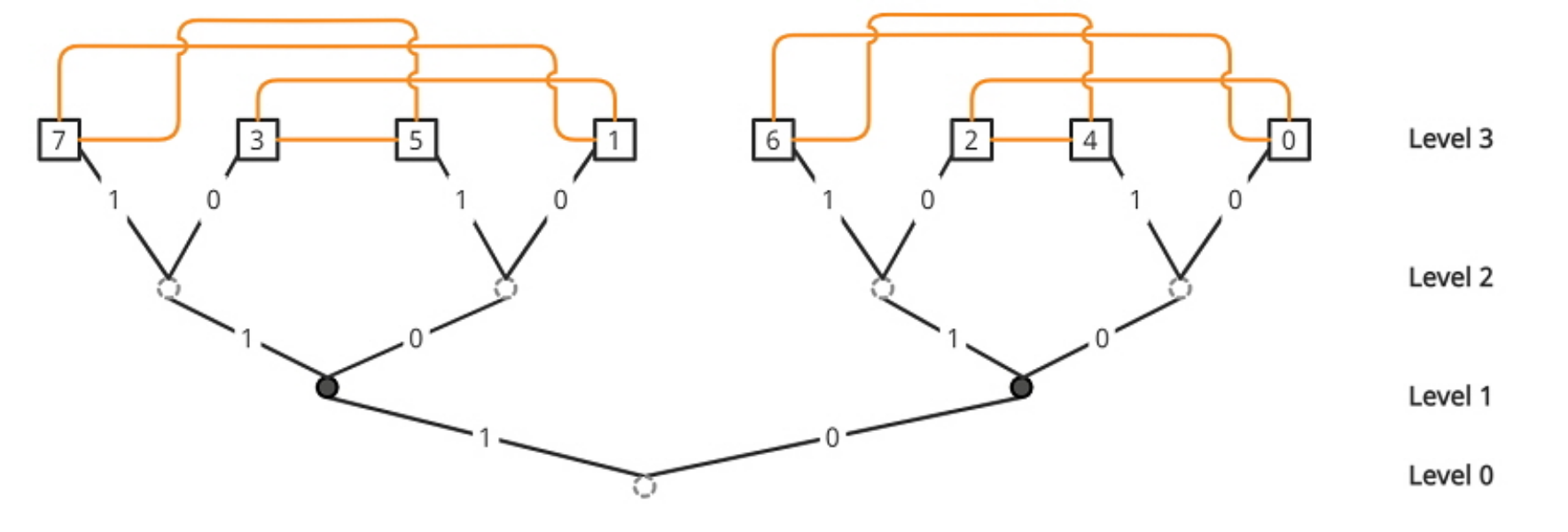}%
\caption{A discrete $2$-adic CNN with $8$ cells: $C_{3}%
=\{0,1,2,3,4,5,7\}\subset\mathbb{Z}_{2}/2^{3}\mathbb{Z}_{2}\subset
2^{-3}\mathbb{Z}_{2}/2^{3}\mathbb{Z}_{2}$. We set $\mathbb{B}=0$ and
$\mathbb{A}(i,j)=\left[  a_{i,j}\right]  $, with $a_{i,j}\neq0$ if
$|i-j|_{2}=1/2$ and $i$, $j\in C_{3}$; $a_{i,j}=0$ otherwise. Taken from
\cite{Zambrano-Zuniga-1}}%
\label{Figure 3}%
\end{center}
\end{figure}
%EndExpansion

\subsection{Stability of $p$-adic continuous CNN}

We assume that $A\left(  \left\vert x\right\vert _{p}\right)  ,B\left(
\left\vert x\right\vert _{p}\right)  \in L^{1}(\mathbb{Z}_{p})$ are radial
functions and that $U$, $Z\in\mathcal{X}_{\infty}$.

For $g\in\mathcal{X}_{\infty}$, set
\[
\boldsymbol{H}(g):=%
%TCIMACRO{\dint \limits_{\mathbb{Z}_{p}}}%
%BeginExpansion
{\displaystyle\int\limits_{\mathbb{Z}_{p}}}
%EndExpansion
A(\left\vert x-y\right\vert _{p})f\left(  g(y)\right)  dy+%
%TCIMACRO{\dint \limits_{\mathbb{Z}_{p}}}%
%BeginExpansion
{\displaystyle\int\limits_{\mathbb{Z}_{p}}}
%EndExpansion
B(\left\vert x-y\right\vert _{p})U(y)dy+Z(x).
\]
Then $\boldsymbol{H}:\mathcal{X}_{\infty}\rightarrow\mathcal{X}_{\infty}$ is a
well-defined \ operator satisfying%
\[
\Vert\boldsymbol{H}(g)-\boldsymbol{H}(g^{\prime})\Vert_{\infty}\leq L(f)\Vert
A\Vert_{1}\Vert g-g^{\prime}\Vert_{\infty}\text{, for }g\text{, }g^{\prime}%
\in\mathcal{X}_{\infty}\text{,}%
\]
where $L(f)$ is the Lipschitz constant of $f$, cf. \cite[Lemma 2]%
{Zambrano-Zuniga-1}.

The following result establishes the existence and uniqueness of a solution
for the Cauchy problem associated with equation (\ref{Continuous_CNN}).

\begin{proposition}
\cite[Proposition 1]{Zambrano-Zuniga-1} Let $\tau$ be a fixed positive real
number. Then, for each $X_{0}\in\mathcal{X}_{\infty}$ there exists a unique
$X\in C([0,\tau],\mathcal{X}_{\infty})$ which satisfies
\begin{equation}
X(x,t)=e^{-t}X_{0}\left(  x\right)  +\int_{0}^{t}e^{-(t-s)}\boldsymbol{H}%
(X(x,s))ds \label{eq:solution}%
\end{equation}
where
\begin{equation}
\boldsymbol{H}X(x,t)=%
%TCIMACRO{\dint \limits_{\mathbb{Z}_{p}}}%
%BeginExpansion
{\displaystyle\int\limits_{\mathbb{Z}_{p}}}
%EndExpansion
A(\left\vert x-y\right\vert _{p})f(X(y,t))dy+%
%TCIMACRO{\dint \limits_{\mathbb{Z}_{p}}}%
%BeginExpansion
{\displaystyle\int\limits_{\mathbb{Z}_{p}}}
%EndExpansion
B(\left\vert x-y\right\vert _{p})U(y)dy+Z(x). \label{eq:Op_H}%
\end{equation}
The function $X(x,t)$ is differentiable in $t$ for all $x$, and it is a
solution of equation (\ref{Continuous_CNN}) with initial datum $X_{0}$.
\end{proposition}

The following result shows the stability of the $p$-adic continuous CNNs.

\begin{theorem}
\cite[Theorem 2]{Zambrano-Zuniga-1} All the states $X(x,t)$ of a $p$-adic
continuous CNN are bounded for all time $t\geq0$. More precisely, if
\[
X_{\max}:=\Vert X_{0}\Vert_{\infty}+\Vert f\Vert_{\infty}\Vert A\Vert
_{1}+\Vert U\Vert_{\infty}\Vert B\Vert_{1}+\Vert Z\Vert_{\infty},
\]
then
\begin{equation}
|X(x,t)|\leq X_{\max}\text{ for all }t\geq0\text{ and for all }x\in
\mathbb{Z}_{p}. \label{No_Blow_up1}%
\end{equation}

In addition%
\[
X_{-}\left(  x\right)  :=\lim\inf_{t\rightarrow\infty}X(x,t)\leq
X(x,t)\leq\lim\sup_{t\rightarrow\infty}X(x,t)=:X_{+}\left(  x\right)  ,
\]
for $x\in\mathbb{Z}_{p}$. If $X_{-}\left(  x\right)  =X_{+}\left(  x\right)
:=X^{\ast}(x)$,then $X^{\ast}(x)$ is a stationary solution of the
CNN$(A,B,U,Z)$ and
\begin{equation}
X^{\ast}(x)\geq-\left\Vert f\right\Vert _{\infty}\Vert A\Vert_{1}-\Vert
U\Vert_{\infty}\Vert B\Vert_{1}+Z(x). \label{Stat_Solution_2}%
\end{equation}

\end{theorem}

\begin{remark}
Condition (\ref{No_Blow_up1}) implies that $X(x,t)$ does not blow-up at finite
time. The existence of a stationary state $X^{\ast}(x)$ means that the state
of each cell of a $p$-adic continuous CNN most settle at stable equilibrium
point after the transient has decayed to zero.
\end{remark}

\subsection{A Numerical simulation}

For an in-depth discussion of the numerical methods use in the numerical
simulations given in this article, the reader may consult
\cite{Zambrano-Zuniga-1}-\cite{Zambrano-Zuniga-2}.

The discretization of the kernels $A$, and $B$ are functions on $G_{l}\times
G_{l}$, while the input $U$ and initial condition $X_{0}$ are functions on
$G_{l}$. We use heat maps to visualize these functions. Since $G_{l}$ is
isomorphic to $\mathbb{Z}/p^{l}\mathbb{Z}$, we identify the leaves of the tree
with the set $\{0,1,\ldots,p^{l}-1\}$. We also include a tree plot that
represents $G_{l}$.

In this example, we take $l=5$, $p=3$, which means we use a tree with
$3^{5}=243$ leaves and $5$ levels. A basic application of the classical CNNs
is image processing, \cite{Chua-Tamas}. In this example we present a
one-dimensional edge detector, which is a $p$-adic, one-dimensional analog of
the examples 3.1 and 3.2 in \cite{Chua-Tamas}. The input (the image) is%
\[
U(x)=\cos(6\pi M(x)),
\]
where $M(x):\mathbb{Z}_{p}\longrightarrow\lbrack0,1]$ is the Monna map. As in
\cite{Chua-Tamas}, we take $X_{0}(x)=0$. To construct templates $A$ and $B$,
we identify a matrix with a test function. We use%
\[
A(x)=2\Omega(3^{5}|x|_{p}),\text{ \ }B(x)=3^{5}\left(  3^{3}\Omega
(3^{5}|x|_{p})-\Omega(3^{3}|x|_{p})\right)  .
\]

Finally, we take $Z(x)=-\Omega(3^{5}|x|_{p})$, $f(x)=0.5(|x+1|-|x-1|)$. The
output $Y(x,t)$ consists of the edges on the input $U$; see Figures
\ref{Figure 4}-\ref{Figure 6}.%

%TCIMACRO{\FRAME{ftbpFU}{5.2235in}{2.7942in}{0pt}{\Qcb{Heat map of $U(x)$. The
%position of each neuron corresponds with a leave of $G_{5}$. Time $25$ and
%step $\delta_{t}$ $=0.05$.}}{\Qlb{Figure 4}}{Figure 4}%
%{\special{ language "Scientific Word";  type "GRAPHIC";  display "USEDEF";
%valid_file "T";  width 5.2235in;  height 2.7942in;  depth 0pt;
%original-width 9.986in;  original-height 4.542in;  cropleft "0";
%tempfilename 'SEUKVU00.wmf';tempfile-properties "XPR";}} }%
%BeginExpansion
\begin{figure}
[h]
\begin{center}
\includegraphics[width=0.75\textwidth]%
{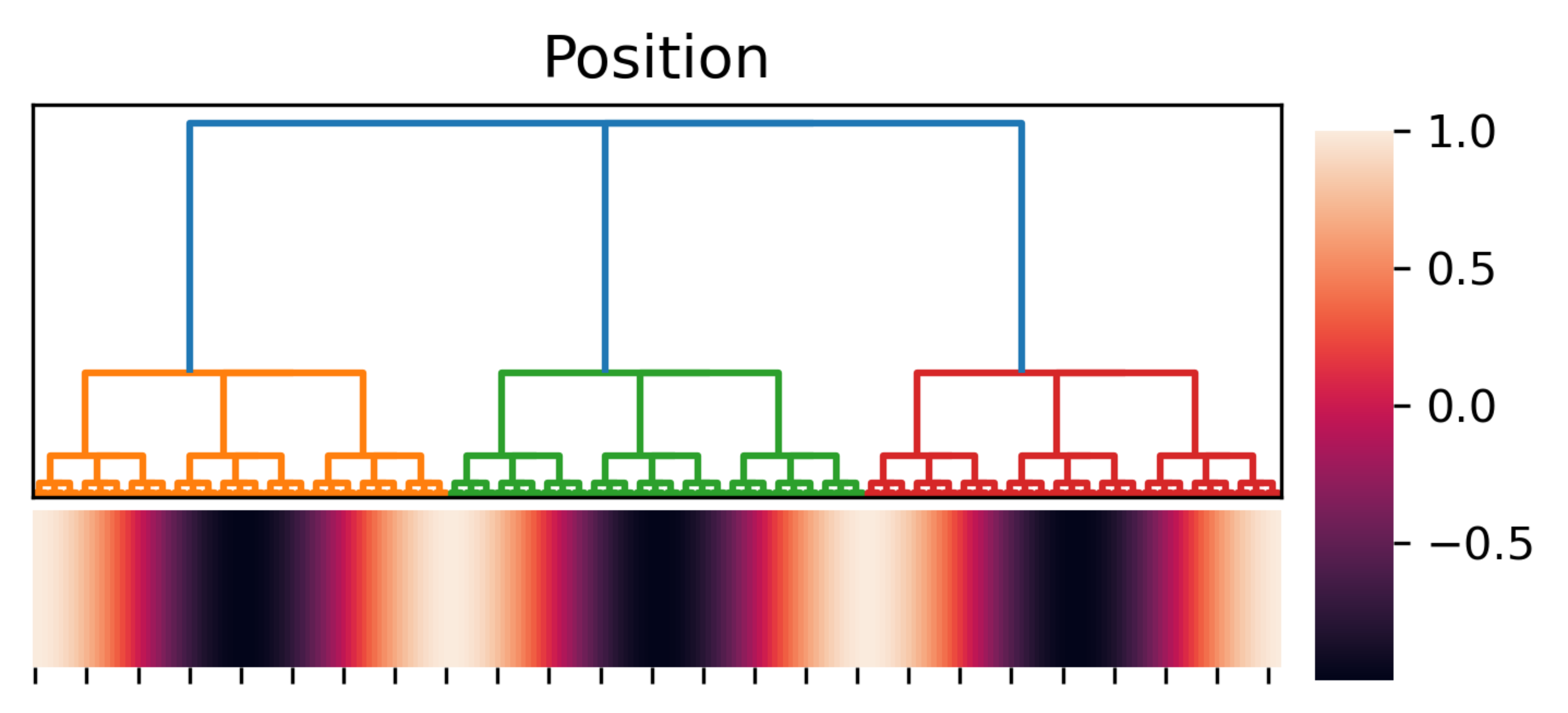}%
\caption{Heat map of $U(x)$. The position of each neuron corresponds with a
leave of $G_{5}$. Time $25$ and step $\delta_{t}$ $=0.05$.}%
\label{Figure 4}%
\end{center}
\end{figure}
%EndExpansion
%TCIMACRO{\FRAME{ftbpFU}{4.5057in}{2.2883in}{0pt}{\Qcb{Heat map of $X(x,t)$.
%Time $25$ and step $\delta_{t}$ $=0.05$.}}{\Qlb{Figurw 5}}{Figure
%5}{\special{ language "Scientific Word";  type "GRAPHIC";
%maintain-aspect-ratio TRUE;  display "USEDEF";  valid_file "T";
%width 4.5057in;  height 2.2883in;  depth 0pt;  original-width 11.8583in;
%original-height 6.0044in;  cropleft "0";  croptop "1";  cropright "1";
%cropbottom "0";  tempfilename 'SEV76Q00.wmf';tempfile-properties "XPR";}} }%
%BeginExpansion
\begin{figure}
[h]
\begin{center}
\includegraphics[width=0.75\textwidth]%
{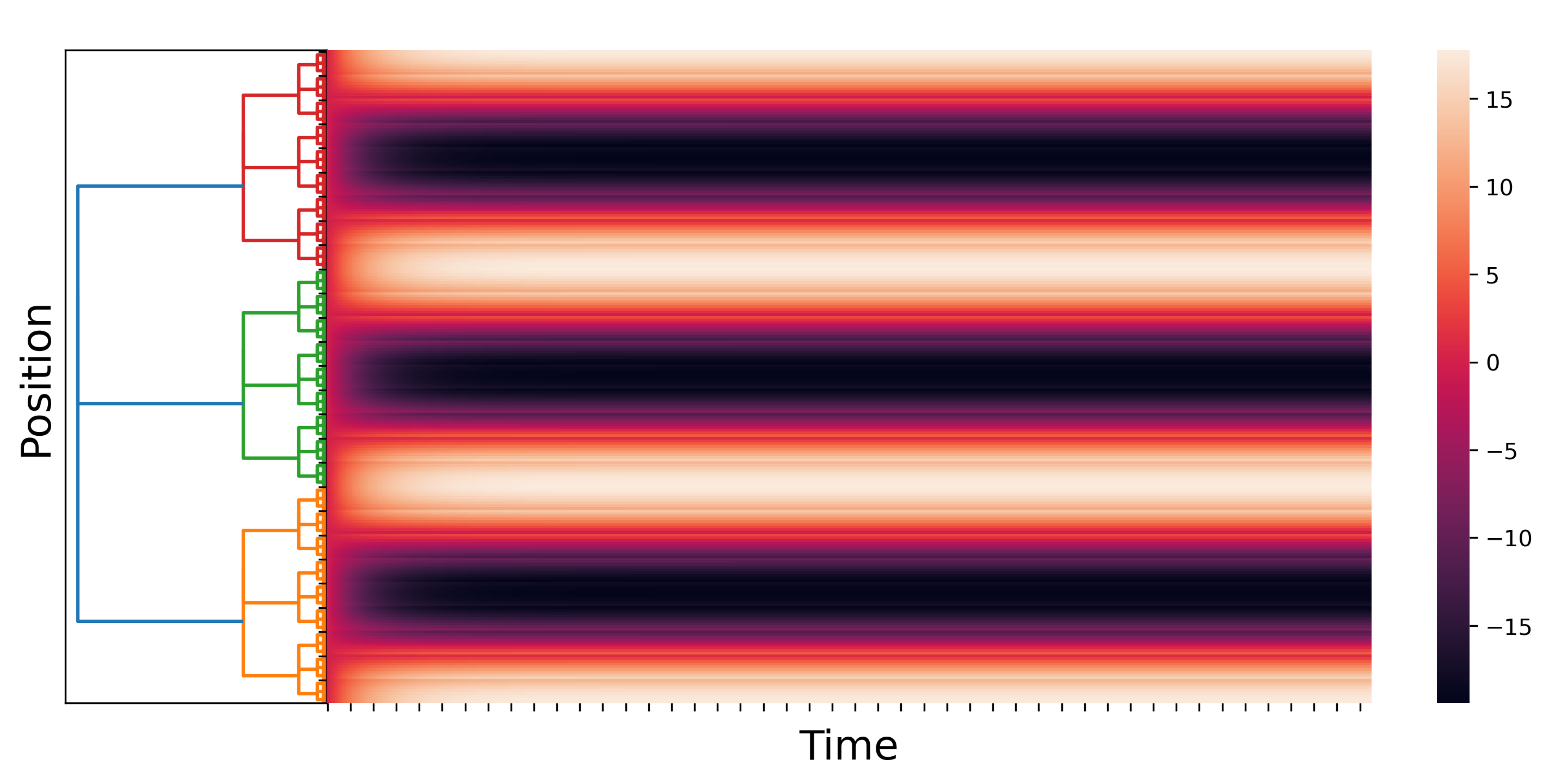}%
\caption{Heat map of $X(x,t)$. Time $25$ and step $\delta_{t}$ $=0.05$.}%
\label{Figurw 5}%
\end{center}
\end{figure}
%EndExpansion
%

%TCIMACRO{\FRAME{ftbpFU}{4.7469in}{2.3782in}{0pt}{\Qcb{Heat map of $Y(x,t)$.
%Time $25$ and step $\delta_{t}$ $=0.05$.}}{\Qlb{Figure 6}}{Figure
%6}{\special{ language "Scientific Word";  type "GRAPHIC";
%maintain-aspect-ratio TRUE;  display "USEDEF";  valid_file "T";
%width 4.7469in;  height 2.3782in;  depth 0pt;  original-width 12.0252in;
%original-height 6.0044in;  cropleft "0";  croptop "1";  cropright "1";
%cropbottom "0";  tempfilename 'SEV79R01.wmf';tempfile-properties "XPR";}} }%
%BeginExpansion
\begin{figure}
[h]
\begin{center}
\includegraphics[width=0.75\textwidth]%
{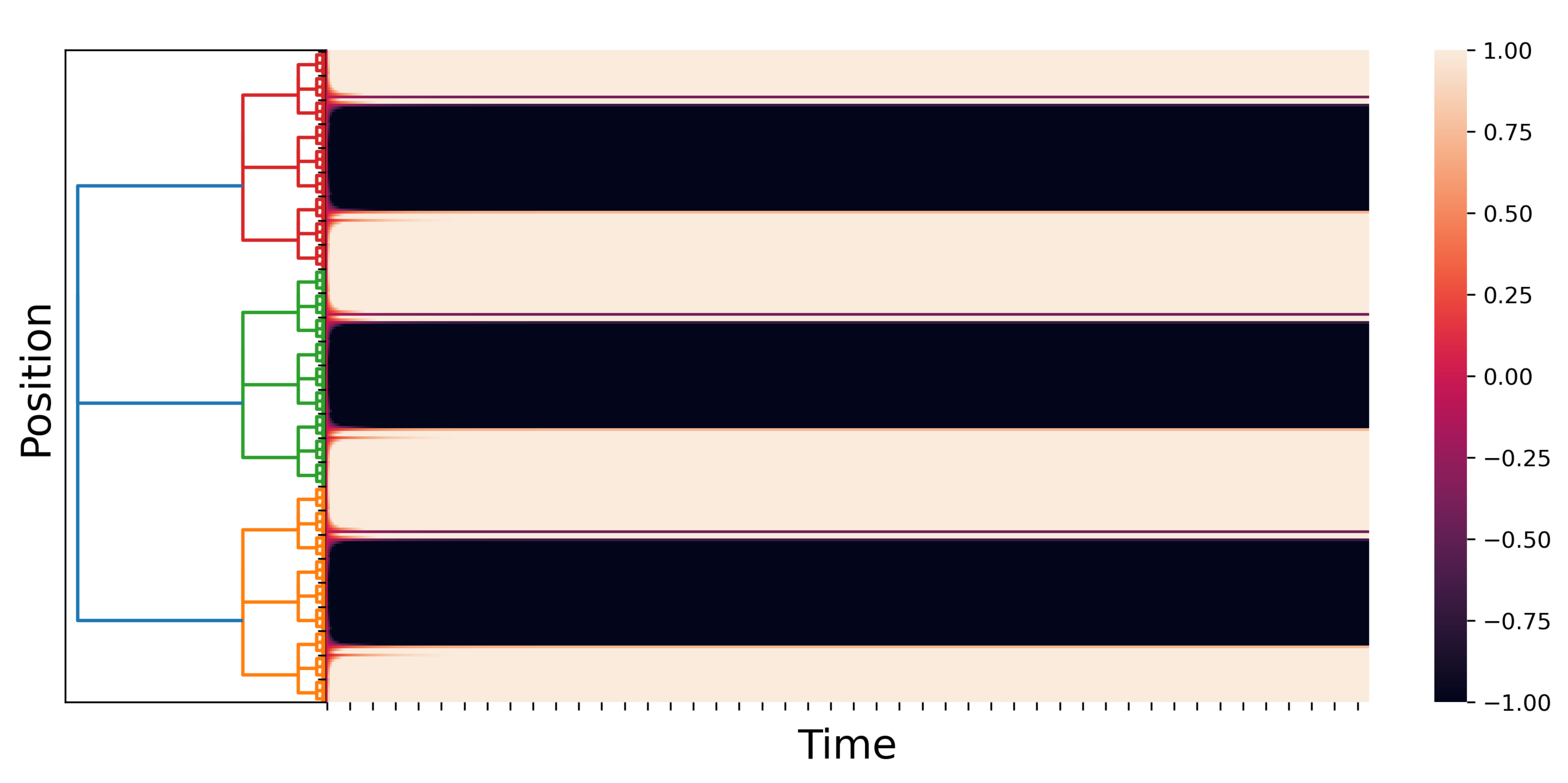}%
\caption{Heat map of $Y(x,t)$. Time $25$ and step $\delta_{t}$ $=0.05$.}%
\label{Figure 6}%
\end{center}
\end{figure}
%EndExpansion

\section{\label{Section_CNNS with delay}Reaction-diffusion CNNs with delay}

In this section, we present an extension of the $p$-adic contimuous CNNs. The
new networks include a $p$-adic diffusion term and delay.

\section{Heat-type equations on the unit ball}

We fix a function $J:\mathbb{Z}_{p}\rightarrow\left[  0,\infty\right)  $
satisfying $\int_{\mathbb{Z}_{p}}J(x)dx=1$. We attach to this function the
operator:%
\[
\boldsymbol{J}f(x)=\int\limits_{\mathbb{Z}_{p}}J(x-y)\left\{
f(y)-f(x)\right\}  dy\text{.}%
\]
We now consider the equation%
\begin{equation}
\frac{\partial u\left(  x,t\right)  }{\partial t}=\boldsymbol{J}u\left(
x,t\right)  . \label{Heat_Equation}%
\end{equation}
This equation describes an ultradiffusion (or $p$-adic diffusion) process in
$\mathbb{Z}_{p}$. Considering $u\left(  x,t\right)  $ as a density of
individuals or particles at the point $x$ and interpreting $J(x-y)$ as the
probability distribution of jumping from location $y$ to location $x$, the
amount $\int_{\mathbb{Z}_{p}}J(x-y)u(y,t)dy$ is the rate at which particles
are arriving at position $x$ from all other places, and $-u(x,t)=-\int
_{\mathbb{Z}_{p}}J(x-y)u(x,t)dy$ is the rate at which they are leaving the
location $x$ to travel to other places. Assuming that there are no creation or
inhalation of particles, i.e.,\ in the absence of external or internal
sources, the density $u(x,t)$ must satisfy the differential equation
(\ref{Heat_Equation}). Equations of type appeared in several $p$-adic models
of complex systems, see, e.g., \cite{KKZuniga}, \cite{Zuniga-PHYA}%
-\cite{Zuniga-Nonlinearity}, \cite{Torresblnaca-Zuniga}-\cite{Bendikov}.
Nowadays, there is a very general theory for equations of type
(\ref{Heat_Equation}). Here, we need a particular case of Theorem 3.1 from
\cite{Zuniga-PHYA}:

\begin{theorem}
Consider the Cauchy problem:%
\begin{equation}
\left\{
\begin{array}
[c]{l}%
u\left(  x,t\right)  \in C^{1}\left(  \left[  0,T\right]  ,C\left(
\mathbb{Z}_{p},\mathbb{R}\right)  \right) \\
\\
\frac{du\left(  x,t\right)  }{dt}=\boldsymbol{J}u(x,t),t\in\left[  0,T\right]
,x\in\mathbb{Z}_{p}\\
\\
u\left(  x,0\right)  =u_{0}\left(  x\right)  \in C\left(  \mathbb{Z}%
_{p},\mathbb{R}_{+}\right)  ,
\end{array}
\right.  \label{Eq_Cauchy_problem}%
\end{equation}
where $T\in\left[  0,\infty\right]  $.There exists a probability measure
$p_{t}\left(  x,\cdot\right)  $, $t\in\left[  0,T\right]  $, with $T=T(u_{0}%
)$, $x\in\mathbb{Z}_{p}$, on the Borel $f$-algebra of $\mathbb{Z}_{p}$, such
that the Cauchy problem (\ref{Eq_Cauchy_problem}) has a unique solution of the
form%
\[
h(x,t)=\int\limits_{\mathbb{Z}_{p}}u_{0}(y)p_{t}\left(  x,dy\right)  .
\]
In addition, $p_{t}\left(  x,\cdot\right)  $ is the transition function of a
Markov process $\mathfrak{X}$ whose paths are right continuous and have no
discontinuities other than jumps.
\end{theorem}

\section{$p$-Adic reaction-diffusion CNNs with delay}

In this section, we use a sigmoidal function $f:\mathbb{R}\rightarrow
\mathbb{R}$, which is a continuous function satisfying: (i) $\lim
_{t\rightarrow\pm\infty}f\left(  t\right)  $ exists; (ii) $f$ is globally
Lipschitz function, i.e., $\left\vert f\left(  t_{2}\right)  -f\left(
t_{1}\right)  \right\vert \leq L\left\vert t_{2}-t_{1}\right\vert $, where $L$
is a positive constant; (iii) $f\left(  0\right)  =0$.

Sigmoidal functions of type $f\left(  t\right)  =\frac{1}{2}\left(
|t+1|-|t-1|\right)  $ are widely used in the framework of CNNs. Notice that
$f$ is bounded.

Given $A(x,y)$, $B(x,y)\in C(\mathbb{Z}_{p}\times\mathbb{Z}_{p},\mathbb{R})$,
and $U\left(  x\right)  $, $Z\left(  x\right)  \in C(\mathbb{Z}_{p}%
,\mathbb{R})$, $\lambda\geq0$, a $p$-adic reaction-diffusion CNN, denoted as
$CNN(J,f,\lambda,A,B,U,Z)$, is the dynamical system given by the following
system integrodifferential equations: (i) state equation:%
\begin{gather}
\frac{\partial X(x,t)}{\partial t}=-\lambda X(x,t)+\boldsymbol{J}X(x,t)+%
%TCIMACRO{\dint \limits_{\mathbb{Z}_{p}}}%
%BeginExpansion
{\displaystyle\int\limits_{\mathbb{Z}_{p}}}
%EndExpansion
A(x,y)Y(y,t+\theta)dy\label{CNN_1}\\
+%
%TCIMACRO{\dint \limits_{\mathbb{Z}_{p}}}%
%BeginExpansion
{\displaystyle\int\limits_{\mathbb{Z}_{p}}}
%EndExpansion
B(x,y)U(y)dy+Z(x),\nonumber
\end{gather}
for $x\in\mathbb{Z}_{p},t\geq0$, $\theta\in\left[  -r,0\right]  $, for some
$r>0$, and $X(x,s)=\varphi\left(  x,s\right)  $, $x\in\mathbb{Z}_{p}$, $-r\leq
s\leq0$, and (ii) output equation%
\begin{equation}
Y(x,t+\theta)=\left\{
\begin{array}
[c]{ll}%
f\left(  X\left(  x,t+\theta\right)  \right)  , & x\in\mathbb{Z}_{p},t\geq0\\
& \\
f\left(  \varphi\left(  x,s\right)  \right)  & x\in\mathbb{Z}_{p},-r\leq
s\leq0,
\end{array}
\right.  \label{CNN_2}%
\end{equation}
where $\varphi\left(  x,s\right)  \in C\left(  \mathbb{Z}_{p}\times\left[
-r,0\right]  ,\mathbb{R}\right)  $. We say that $X(x,t)\in\mathbb{R}$ is the
\textit{state of cell} $x$ at the time $t$, $Y(x,t)\in\mathbb{R}$ is
\textit{the output of cell} $x$ at the time $t$. Function $A(x,y)$ is
\textit{the kernel of the feedback operator, while \ function } $B(x,y)$ is
t\textit{he kernel of the feedforward operator}. Function $U$ is \textit{the
input of the $\text{CNN}$}, while function $Z$ is \textit{the threshold of the
$\text{CNN. The parameter }$}$\theta\in\left[  -r,0\right]  $ gives the time delay.

In the case $A(x,y)=A(\left\vert x-y\right\vert _{p})$ and
$B(x,y)=B(\left\vert x-y\right\vert _{p})$, we say that the CNN is space invariant.

\subsection{\label{Section_discretization}Discrete versions}

To obtain a discretization of (\ref{CNN_1})-(\ref{CNN_2}), we fix a time
horizon $\tau\in\left(  0,\infty\right)  $, and take
\begin{equation}
X(x,t)=\sum_{i\in G_{l}}X(i,t)\Omega\left(  p^{l}|x-i|_{p}\right)  ,\text{
}X(i,t)\in C^{1}(\left[  -r,\tau\right]  ,\mathbb{R})\text{ for }i\in G_{l},
\label{Formula_1}%
\end{equation}%
\begin{equation}
Y\left(  x,t\right)  =f\left(  X(x,t)\right)  =\sum_{i\in G_{l}}f\left(
X(i,t)\right)  \Omega\left(  p^{l}|x-i|_{p}\right)  , \label{Formula_1A}%
\end{equation}%
\begin{equation}
J(x)=\sum_{i\in G_{l}}J(i)\Omega\left(  p^{l}|x-i|_{p}\right)  ,
\label{Formula_2}%
\end{equation}%
\begin{equation}
A\left(  x,y\right)  =\sum_{k\in G_{l}}\sum_{j\in G_{l}}A(k,j)\Omega\left(
p^{l}|x-k|_{p}\right)  \Omega\left(  p^{l}|y-j|_{p}\right)  ,\text{ }
\label{Formula_3}%
\end{equation}%
\begin{equation}
B\left(  x,y\right)  =\sum_{k\in G_{l}}\sum_{j\in G_{l}}B(k,j)\Omega\left(
p^{l}|x-k|_{p}\right)  \Omega\left(  p^{l}|y-j|_{p}\right)  ,
\label{Formula_4}%
\end{equation}%
\begin{equation}
Z(x)=\sum_{k\in G_{l}}Z(k)\Omega\left(  p^{l}|x-k|_{p}\right)  ,\text{
\ }U(x)=\sum_{k\in G_{l}}U(k)\Omega\left(  p^{l}|x-k|_{p}\right)  .
\label{Formula_5}%
\end{equation}
Then, by using (ii) For very
\[
\Omega\left(  p^{l}|x-i|_{p}\right)  \ast\Omega\left(  p^{l}|x-j|_{p}\right)
=p^{-l}\Omega\left(  p^{l}|x-\left(  i+j\right)  |_{p}\right)  ,\text{ for
}i,j\in G_{l},\text{,}%
\]
and formulas (\ref{Formula_1})-(\ref{Formula_2}),
\begin{gather*}
-\lambda X(x,t)+\boldsymbol{J}X(x,t)=J\ast X(x,t)-\left(  \lambda+1\right)
X(x,t)\\
=\sum_{i\in G_{l}}\sum_{j\in G_{l}}p^{-l}J(i)X(j,t)\Omega\left(
p^{l}|x-\left(  i+j\right)  |_{p}\right)  -\left(  \lambda+1\right)
\sum_{k\in G_{l}}X(k,t)\Omega\left(  p^{l}|x-k|_{p}\right) \\
=\sum_{k\in G_{l}}\sum_{i\in G_{l}}p^{-l}J(i)X(k-i,t)\Omega\left(
p^{l}|x-k|_{p}\right)  -\left(  \lambda+1\right)  \sum_{k\in G_{l}%
}X(k,t)\Omega\left(  p^{l}|x-k|_{p}\right) \\
=\sum_{k\in G_{l}}\left\{  \sum_{i\in G_{l}}p^{-l}J(i)X(k-i,t)-\left(
\lambda+1\right)  X(k,t)\right\}  \Omega\left(  p^{l}|x-k|_{p}\right)  ,
\end{gather*}
and by using (\ref{Formula_1A}) and (\ref{Formula_3}),%
\begin{gather*}%
%TCIMACRO{\dint \limits_{\mathbb{Z}_{p}}}%
%BeginExpansion
{\displaystyle\int\limits_{\mathbb{Z}_{p}}}
%EndExpansion
A(x,y)Y(y,t+\theta)dy=\\
\sum_{k\in G_{l}}\left\{  \sum_{j\in G_{l}}A(k,j)f\left(  X(j,t+\theta
)\right)
%TCIMACRO{\dint \limits_{\mathbb{Z}_{p}}}%
%BeginExpansion
{\displaystyle\int\limits_{\mathbb{Z}_{p}}}
%EndExpansion
\Omega\left(  p^{l}|y-j|_{p}\right)  dy\right\}  \Omega\left(  p^{l}%
|x-k|_{p}\right) \\
=\sum_{k\in G_{l}}\left\{  \sum_{j\in G_{l}}p^{-l}A(k,j)f\left(
X(j,t+\theta)\right)  \right\}  \Omega\left(  p^{l}|x-k|_{p}\right)  ,
\end{gather*}
and by using (\ref{Formula_4})-(\ref{Formula_5}),
\[%
%TCIMACRO{\dint \limits_{\mathbb{Z}_{p}}}%
%BeginExpansion
{\displaystyle\int\limits_{\mathbb{Z}_{p}}}
%EndExpansion
B(x,y)U(y)dy+Z(x)=\sum_{k\in G_{l}}\left\{  \sum_{i\in G_{l}}p^{-l}%
B(k,i)U(k)+Z(k)\right\}  \Omega\left(  p^{l}|x-k|_{p}\right)  .
\]
Finally, the discretization of (\ref{CNN_1})-(\ref{CNN_2}) takes the following
form:%
\begin{equation}
\left\{
\begin{array}
[c]{l}%
\frac{\partial}{\partial t}X(k,t)=-\left(  \lambda+1\right)  X(k,t)+\sum_{i\in
G_{l}}p^{-l}J(i)X(k-i,t)\\
\\
\text{ \ \ \ \ \ \ \ \ \ \ \ \ \ \ \ }+\sum_{i\in G_{l}}p^{-l}A(k,i)Y\left(
i,t+\theta\right)  \text{\ }+\sum_{i\in G_{l}}p^{-l}B(k,i)U(i)+Z(k)\\
\\
Y\left(  k,t+\theta\right)  =f\left(  X(k,t+\theta)\right)  ,
\end{array}
\right.  \label{System_1}%
\end{equation}
for $k\in G_{l}$.

Using matrix notation:%
\[
\boldsymbol{X}(t)=\left[  X(k,t)\right]  _{k\in G_{l}}\text{, \ }%
\boldsymbol{Y}(t+\theta)=\left[  Y(k,t+\theta)\right]  _{k\in G_{l}}\text{,
}\boldsymbol{J}=\left[  J(k)\right]  _{k\in G_{l}}\text{,}%
\]%
\[
\text{ }\boldsymbol{U}=\left[  U(k)\right]  _{k\in G_{l}}\text{,\ }%
\boldsymbol{Z}=\left[  Z(k)\right]  _{k\in G_{l}}\text{, }\boldsymbol{A}%
=\left[  A(k,i)\right]  _{k,i\in G_{l}}\text{, \ }\boldsymbol{B}=\left[
B(k,i)\right]  _{k,i\in G_{l}},
\]
the system (\ref{System_1}) can be rewritten as%
\[
\left\{
\begin{array}
[c]{l}%
\frac{\partial}{\partial t}\boldsymbol{X}(t)=-\left(  \lambda+1\right)
\boldsymbol{X}(t)+p^{-l}\boldsymbol{J}\ast\boldsymbol{X}(t)+p^{-l}%
\boldsymbol{AY}(t+\theta)+\boldsymbol{BU}+\boldsymbol{Z}\\
\\
\boldsymbol{Y}(t+\theta)=\left[  f\left(  X(k,t+\theta)\right)  \right]
_{k\in G_{l}}.
\end{array}
\right.
\]

\subsection{A class of reaction-diffusion\ equations with delay}

Given a Banach space $\left(  \mathcal{X},\left\Vert \cdot\right\Vert
_{\mathcal{X}}\right)  $ and $r>0$, we denote by $C(\left[  -r,0\right]
,\mathcal{X})$, the Banach space of $\mathcal{X}$-valued functions on $\left[
-r,0\right]  $ with the supremum norm:%
\begin{equation}
\left\Vert f\right\Vert =\sup_{s\in\left[  -r,0\right]  }\left\Vert f\left(
s\right)  \right\Vert _{\mathcal{X}}. \label{norm_1}%
\end{equation}
In the case $\mathcal{X}=\mathbb{R}$, \ we denote\ norm (\ref{norm_1}) as
$\left\Vert f\right\Vert _{\infty}$. We set $C(\mathbb{Z}_{p},\mathbb{R})$ for
the Banach space of real valued functions on $\mathbb{Z}_{p}$ with the
supremum norm:%
\[
\left\Vert h\right\Vert _{\infty}=\sup_{x\in\mathbb{Z}_{p}}\left\vert h\left(
x\right)  \right\vert .
\]
The results presented in this section are valid for arbitrary Banach spaces
$\left(  \mathcal{X},\left\Vert \cdot\right\Vert _{\mathcal{X}}\right)  $.
However, we apply the results in two cases:
\[
\left(  \mathcal{X},\left\Vert \cdot\right\Vert _{\mathcal{X}}\right)
=\left(  C(\mathbb{Z}_{p},\mathbb{R}),\left\Vert \cdot\right\Vert _{\infty
}\right)  \text{, \ \ }\left(  \mathcal{D}^{l}\left(  \mathbb{Z}_{p}\right)
,\left\Vert \cdot\right\Vert _{\infty}\right)  =\left(  \mathcal{D}^{l}\left(
\mathbb{Z}_{p},\mathbb{R}\right)  ,\left\Vert \cdot\right\Vert _{\infty
}\right)  \simeq\left(  \mathbb{R}^{p^{l}},\left\Vert \cdot\right\Vert
\right)  .
\]
Notice that
\[
\left(  \mathcal{D}^{l}\left(  \mathbb{Z}_{p},\mathbb{R}\right)  ,\left\Vert
\cdot\right\Vert _{\infty}\right)  \hookrightarrow\left(  C(\mathbb{Z}%
_{p},\mathbb{R}),\left\Vert \cdot\right\Vert _{\infty}\right)  ,
\]
where the arrow denotes a continuous embedding.

Along this section,\ we work with `delayed functions' of type $u\left(
x,t+\theta\right)  $, where $u$ is a continuous function, $x\in\mathbb{Z}_{p}%
$, $t\in\left[  a,b\right]  $, with $a<b$, and $\theta\in\left[  -r,0\right]
$. We identify this function with an element of $C(\left[  -r,0\right]
,C(\mathbb{Z}_{p}\times\left[  a-r,b\right]  ,\mathbb{R}))$ parametrized by
$x\in\mathbb{Z}_{p}$ given by%
\[
u_{t}\left(  x\right)  (\theta)=u\left(  x,t+\theta\right)  \text{ for }%
\theta\in\left[  -r,0\right]  \text{, }t\geq0\text{.}%
\]
Then, $u:\left[  a-r,b\right]  \rightarrow C(\mathbb{Z}_{p},\mathbb{R})$ is a
continuous function.

We denote by $\boldsymbol{D}$ the infinitesimal generator of a $C_{0}%
$-semigroup $\left\{  \mathcal{T}_{t}\right\}  _{t\geq0}$ on $C(\mathbb{Z}%
_{p},\mathbb{R})$.

We now identify $C(\mathbb{Z}_{p}\times\left[  -r,0\right]  ,\mathbb{R}%
)\ $with $\mathcal{C}:=C\left(  \left[  -r,0\right]  ,C(\mathbb{Z}%
_{p},\mathbb{R})\right)  $ and set%
\[%
\begin{array}
[c]{cccc}%
F: & C\left(  \left[  -r,0\right]  ,C(\mathbb{Z}_{p},\mathbb{R})\right)  &
\rightarrow & C(\mathbb{Z}_{p},\mathbb{R})\\
&  &  & \\
& \phi & \rightarrow & F(x,\phi\left(  x,\cdot\right)  ).
\end{array}
\]
We assume that $F$ is globally Lipschitz in $C(\mathbb{Z}_{p},\mathbb{R)}$,
that is,
\begin{equation}
\left\Vert F(\phi)-F(\psi)\right\Vert _{C(\mathbb{Z}_{p},\mathbb{R)}}\leq
L_{0}\left\Vert \phi-\psi\right\Vert _{\mathcal{C}},
\label{Condition_Lipschitz_2}%
\end{equation}
for $\phi,\psi\in\mathcal{C}$.

We now consider the following system of reaction-diffusion equations with
delay on the $p$-adic unit ball:%
\[
\left\{
\begin{array}
[c]{l}%
\frac{\partial u(x,t)}{\partial t}=\boldsymbol{D}u(x,t)+F(x,u_{t}(x))\text{,
}x\in\mathbb{Z}_{p}\text{, }t>0\\
\\
u(x,s)=\varphi\left(  x,s\right)  \text{, }x\in\mathbb{Z}_{p}\text{, }-r\leq
s\leq0,
\end{array}
\right.
\]
where $\varphi\left(  x,s\right)  \in C(\mathbb{Z}_{p}\times\left[
-r,0\right]  ,\mathbb{R})$. Since the boundary\ of the unit ball
$\mathbb{Z}_{p}$ is the empty set, we do not need boundary conditions.

\begin{proposition}
\label{Proposition_1}Assuming (\ref{Condition_Lipschitz_2}), and that
$\boldsymbol{D}$ is the infinitesimal generator of a $C_{0}$-semigroup
$\left\{  \mathcal{T}_{t}\right\}  _{t\geq0}$ on $C(\mathbb{Z}_{p}%
,\mathbb{R})$. Then, for $t\in\left[  0,\tau\right]  $, with $\tau>0$
arbitrary, and $\varphi\in\mathcal{C}$, there exists a unique continuous
function $u:\left[  -r,\tau\right]  \rightarrow C(\mathbb{Z}_{p},\mathbb{R})$
which is a solution of the following initial value problem:%
\begin{equation}
\left\{
\begin{array}
[c]{ll}%
u(t)=\mathcal{T}\left(  t\right)  \varphi\left(  \cdot,0\right)  +%
%TCIMACRO{\dint \nolimits_{0}^{t}}%
%BeginExpansion
{\displaystyle\int\nolimits_{0}^{t}}
%EndExpansion
\mathcal{T}\left(  t-s\right)  F(u_{s})ds, & 0\leq t\leq\tau\\
& \\
u(0)=\varphi\left(  \cdot,0\right)  . &
\end{array}
\right.  \label{Duhamel_form}%
\end{equation}

\end{proposition}

\begin{proof}
See \cite[Theorem 1.1]{Wu}.
\end{proof}

The\ solution of (\ref{Duhamel_form}) is called a \textit{mild solution}.

We set%
\[
\mathcal{C}_{l}:=C\left(  \left[  -r,0\right]  ,\mathcal{D}^{l}\left(
\mathbb{Z}_{p},\mathbb{R}\right)  \right)  \hookrightarrow\mathcal{C}\text{.}%
\]
Then
\[
\mathcal{C}_{l}=\left\{  f\in\mathcal{C};f\left(  x,\theta\right)  =\sum_{k\in
G_{l}}f\left(  k,\theta\right)  \Omega\left(  p^{l}|x-k|_{p}\right)  ,f\left(
k,\cdot\right)  \in C\left(  \left[  -r,0\right]  ,\mathbb{R}\right)
\right\}  .
\]

\begin{proposition}
\label{Proposition_1A}Assuming\ that%
\[%
\begin{array}
[c]{cccc}%
F: & \mathcal{C}_{l} & \rightarrow & \mathcal{D}^{l}\left(  \mathbb{Z}%
_{p},\mathbb{R}\right) \\
&  &  & \\
& \phi & \rightarrow & f\left(  x,\phi\left(  x,\cdot\right)  \right)  .
\end{array}
\]
$F$ is globally Lipschitz in $\mathcal{D}^{l}\left(  \mathbb{Z}_{p}%
,\mathbb{R}\right)  $, see (\ref{Condition_Lipschitz_2}), and that
$\boldsymbol{D}$ is the infinitesimal generator of a $C_{0}$-semigroup
$\left\{  \mathcal{T}_{t}\right\}  _{t\geq0}$ on $\mathcal{D}^{l}\left(
\mathbb{Z}_{p},\mathbb{R}\right)  $. Then, for $t\in\left[  0,\tau\right]  $,
with $\tau>0$ arbitrary, and $\varphi\left(  \cdot,0\right)  \in
\mathcal{C}_{l}$, there exists a unique continuous function $u:\left[
-r,\tau\right]  \rightarrow\mathcal{D}^{l}\left(  \mathbb{Z}_{p}%
,\mathbb{R}\right)  $ which is a solution of the initial value problem
(\ref{Duhamel_form}).
\end{proposition}

\begin{proof}
See \cite[Theorem 1.1]{Wu}.
\end{proof}

\subsection{The Cauchy Problem}

In this section we study the Cauchy problem (\ref{CNN_1})-(\ref{CNN_2}).

\begin{lemma}
\label{Lemma_1}Using the notation $\mathcal{C}=$ $C\left(  \mathbb{Z}%
_{p}\times\left[  -r,0\right]  ,\mathbb{R}\right)  $ as before. The mapping
\[%
\begin{array}
[c]{cccc}%
F: & \mathcal{C} & \rightarrow & \mathcal{C}\\
&  &  & \\
& \phi\left(  x,\theta\right)  & \rightarrow &
%TCIMACRO{\dint \limits_{\mathbb{Z}_{p}}}%
%BeginExpansion
{\displaystyle\int\limits_{\mathbb{Z}_{p}}}
%EndExpansion
A(x,y)f\left(  \phi\left(  y,\theta\right)  \right)  dy+%
%TCIMACRO{\dint \limits_{\mathbb{Z}_{p}}}%
%BeginExpansion
{\displaystyle\int\limits_{\mathbb{Z}_{p}}}
%EndExpansion
B(x,y)U(y)dy+Z(x).
\end{array}
\]
is well-defined \ and continuous, and satisfies%
\[
\left\vert F(\phi)-F(\psi)\right\vert \leq\left\Vert A\right\Vert _{\infty
}L\left\Vert \phi-\psi\right\Vert _{\infty},
\]
for any $\phi,\psi\in\mathcal{C}$.
\end{lemma}

\begin{proof}
By using that $\left\vert A(x,y)f\left(  \phi\left(  y,\theta\right)  \right)
\right\vert \leq\left\Vert A\right\Vert _{\infty}\left\Vert f\right\Vert
_{\infty}\Omega\left(  \left\vert y\right\vert _{p}\right)  $, $\left\vert
B(x,y)U(y)\right\vert $ $\leq\left\Vert B\right\Vert _{\infty}\left\Vert
U\right\Vert _{\infty}\Omega\left(  \left\vert y\right\vert _{p}\right)  $,
and the dominated convergence theorem, we conclude that $\left(  F\left(
\phi\right)  \right)  \left(  x,\theta\right)  $ is a continuous function.
Now,%
\begin{gather*}
\left\vert F(\phi)-F(\psi)\right\vert \leq%
%TCIMACRO{\dint \limits_{\mathbb{Z}_{p}}}%
%BeginExpansion
{\displaystyle\int\limits_{\mathbb{Z}_{p}}}
%EndExpansion
\left\vert A(x,y\right\vert )\left\vert f\left(  \phi\left(  y,\theta\right)
\right)  -f\left(  \psi\left(  y,\theta\right)  \right)  \right\vert dy\\
\leq\left\Vert A\right\Vert _{\infty}L%
%TCIMACRO{\dint \limits_{\mathbb{Z}_{p}}}%
%BeginExpansion
{\displaystyle\int\limits_{\mathbb{Z}_{p}}}
%EndExpansion
\left\vert \left(  \phi\left(  y,\theta\right)  \right)  -f\left(  \psi\left(
y,\theta\right)  \right)  \right\vert dy\\
\leq\left\Vert A\right\Vert _{\infty}L\sup_{y\in\mathbb{Z}_{p}}\sup_{\theta
\in\left[  -r,0\right]  }\left\vert \left(  \phi\left(  y,\theta\right)
\right)  -f\left(  \psi\left(  y,\theta\right)  \right)  \right\vert
=\left\Vert A\right\Vert _{\infty}L\left\Vert \phi-\psi\right\Vert _{\infty}.
\end{gather*}

\end{proof}

\begin{lemma}
\label{Lemma_2}Assuming that $A(x,y),B(x,y)\in\mathcal{D}^{l}\left(
\mathbb{Z}_{p}\times\mathbb{Z}_{p},\mathbb{R}\right)  $, $U(x),Z(x)\in
\mathcal{D}^{l}\left(  \mathbb{Z}_{p},\mathbb{R}\right)  $, and using the
notation $\mathcal{C}_{l}=$ $C\left(  \left[  -r,0\right]  ,\mathcal{D}%
^{l}\left(  \mathbb{Z}_{p},\mathbb{R}\right)  \right)  $ as before. The
mapping
\[%
\begin{array}
[c]{cccc}%
F: & \mathcal{C}_{l} & \rightarrow & \mathcal{D}^{l}\left(  \mathbb{Z}%
_{p}\times\left[  -r,0\right]  ,\mathbb{R}\right) \\
&  &  & \\
& \phi\left(  x,\theta\right)  & \rightarrow &
%TCIMACRO{\dint \limits_{\mathbb{Z}_{p}}}%
%BeginExpansion
{\displaystyle\int\limits_{\mathbb{Z}_{p}}}
%EndExpansion
A(x,y)f\left(  \phi\left(  y,\theta\right)  \right)  dy+%
%TCIMACRO{\dint \limits_{\mathbb{Z}_{p}}}%
%BeginExpansion
{\displaystyle\int\limits_{\mathbb{Z}_{p}}}
%EndExpansion
B(x,y)U(y)dy+Z(x).
\end{array}
\]
is well-defined \ and continuous, and satisfies%
\[
\left\vert F(\phi)-F(\psi)\right\vert \leq\left\Vert A\right\Vert _{\infty
}L\left\Vert \phi-\psi\right\Vert _{\infty},
\]
for any $\phi,\psi\in\mathcal{C}_{l}$.
\end{lemma}

\begin{proof}
The proof is completely similar to the one given for Lemma \ref{Lemma_1}. To
show that the range of map $F$ is contained in $\mathcal{D}^{l}\left(
\mathbb{Z}_{p}\times\left[  -r,0\right]  ,\mathbb{R}\right)  $, we use the
discretization formulas given in Section \ref{Section_discretization}.
\end{proof}

\begin{theorem}
\label{Theorem_2}With the notation of Propositions \ref{Proposition_1}%
-\ref{Proposition_1A}, taking $\mathcal{X}$ to be $C(\mathbb{Z}_{p}%
,\mathbb{R})$ or $\mathcal{D}^{l}\left(  \mathbb{Z}_{p},\mathbb{R}\right)  $,
and assuming that $\boldsymbol{D}$ is the infinitesimal generator of a $C_{0}%
$-semigroup $\left\{  \mathcal{T}_{t}\right\}  _{t\geq0}$ on $X$. We take
$\varphi\in\mathcal{C}$, if $\mathcal{X}=C(\mathbb{Z}_{p},\mathbb{R})$, and
$\varphi\in\mathcal{C}_{l}$, if $\mathcal{X}=\mathcal{D}^{l}\left(
\mathbb{Z}_{p},\mathbb{R}\right)  $. Then, for $t\in\left[  0,\tau\right]  $,
with $\tau>0$ arbitrary, there exists a unique solution of the initial value
problem%
\begin{equation}
\left\{
\begin{array}
[c]{l}%
u\in C\left(  \left[  -r,\tau\right]  ,Dom(\boldsymbol{D})\right)  \cap
C^{1}\left(  \left[  -r,\tau\right]  ,\mathcal{X}\right) \\
\\
\frac{\partial u(t)}{\partial t}=\boldsymbol{D}u(x,t)+%
%TCIMACRO{\dint \limits_{\mathbb{Z}_{p}}}%
%BeginExpansion
{\displaystyle\int\limits_{\mathbb{Z}_{p}}}
%EndExpansion
A(x,y)f(u_{t}\left(  y,\theta\right)  )dy+\\
\text{ \ \ \ \ \ \ \ \ \ \ \ \ \ \ \ \ \ \ \ \ \ \ \ }%
%TCIMACRO{\dint \limits_{\mathbb{Z}_{p}}}%
%BeginExpansion
{\displaystyle\int\limits_{\mathbb{Z}_{p}}}
%EndExpansion
B(x,y)U(y)dy+Z(x)\text{, }x\in\mathbb{Z}_{p}\text{, }0\leq t\leq\tau\\
\\
u(0)=\phi\left(  x,0\right)  \text{.}%
\end{array}
\right.  \label{Cauchy_1}%
\end{equation}
satisfying (\ref{Duhamel_form}).
\end{theorem}

\begin{proof}
The first step is to show that a solution of (\ref{Cauchy_1}), i.e., that a
classical solution, is a mild solution (\ref{Duhamel_form}). This follows from
Lemmas \ref{Lemma_1}-\ref{Lemma_2}, by using \cite[Theorem 5.1.1]{Milan}. In
the second step, one shows the existence of an unique mild solution, cf.
Propositions \ref{Proposition_1}- \ref{Proposition_1A}. In the third step, one
shows \ that the mild solution (\ref{Duhamel_form}) is differentiable. This
follows from\ the fact that $%
%TCIMACRO{\dint \nolimits_{0}^{t}}%
%BeginExpansion
{\displaystyle\int\nolimits_{0}^{t}}
%EndExpansion
\mathcal{T}\left(  t-s\right)  F(u_{s})ds$ is differentiable, which in turn
follows from the fact that $F(u_{s})$ is integrable,\ see \cite[Corrollary
4.7.5]{Milan}, since \
\begin{align*}
\left\vert F(u_{s})\left(  x\right)  \right\vert  &  =\left\vert \text{ }%
%TCIMACRO{\dint \limits_{\mathbb{Z}_{p}}}%
%BeginExpansion
{\displaystyle\int\limits_{\mathbb{Z}_{p}}}
%EndExpansion
A(x,y)f(u_{s}\left(  y\right)  )dy+%
%TCIMACRO{\dint \limits_{\mathbb{Z}_{p}}}%
%BeginExpansion
{\displaystyle\int\limits_{\mathbb{Z}_{p}}}
%EndExpansion
B(x,y)U(y)dy+Z(x)\right\vert \\
&  \leq\left\Vert A\right\Vert _{\infty}\left\Vert f\right\Vert _{\infty
}+\left\Vert B\right\Vert _{\infty}\left\Vert U\right\Vert _{\infty
}+\left\Vert Z\right\Vert _{\infty}.
\end{align*}
Finally, the uniqueness of the solution of Cauchy problem (\ref{Cauchy_1})
follows from the uniqueness of the mild solution, cf. Propositions
\ref{Proposition_1}-\ref{Proposition_1A}.
\end{proof}

\begin{corollary}
\label{Corollary_1} Set $\mathcal{X}$ to be $C(\mathbb{Z}_{p},\mathbb{R})$ or
$\mathcal{D}^{l}\left(  \mathbb{Z}_{p},\mathbb{R}\right)  $, and
$\boldsymbol{D=J}$, $\mathcal{T}_{t}=e^{t\boldsymbol{J}}$, $t\geq0$, on
$\mathcal{X}$. We take $\varphi\in\mathcal{C}$, if $\mathcal{X}=C(\mathbb{Z}%
_{p},\mathbb{R})$, and $\varphi\in\mathcal{C}_{l}$, if $\mathcal{X}%
=\mathcal{D}^{l}\left(  \mathbb{Z}_{p},\mathbb{R}\right)  $. In this last
case, take $A(x,y),B(x,y)\in\mathcal{D}^{l}\left(  \mathbb{Z}_{p}%
\times\mathbb{Z}_{p},\mathbb{R}\right)  $, $J(x)$, $U(x)$, $Z(x)\in
\mathcal{D}^{l}\left(  \mathbb{Z}_{p},\mathbb{R}\right)  $. Then, for
$t\in\left[  0,\tau\right]  $, with $\tau>0$ arbitrary, there exists a unique
solution of the initial value problem%
\[
\left\{
\begin{array}
[c]{l}%
X\in C^{1}\left(  \left[  -r,\tau\right]  ,\mathcal{X}\ \right) \\
\\
\frac{\partial X(x,t)}{\partial t}=\left(  \boldsymbol{J}-\lambda\right)
X(x,t)+%
%TCIMACRO{\dint \limits_{\mathbb{Z}_{p}}}%
%BeginExpansion
{\displaystyle\int\limits_{\mathbb{Z}_{p}}}
%EndExpansion
A(x,y)f(X\left(  y,t+\theta\right)  )dy\\
\text{ \ \ \ \ \ \ \ \ \ \ \ \ \ \ \ \ \ \ \ \ \ \ \ \ \ }+%
%TCIMACRO{\dint \limits_{\mathbb{Z}_{p}}}%
%BeginExpansion
{\displaystyle\int\limits_{\mathbb{Z}_{p}}}
%EndExpansion
B(x,y)U(y)dy+Z(x)\text{, }x\in\mathbb{Z}_{p}\text{, }t\in\left[
0,\tau\right]  ,\theta\in\left[  -r,0\right] \\
\\
X(x,0)=X_{0}\left(  x,0\right)  \in C\left(  \left[  -r,0\right]
,\mathcal{X}\right)
\end{array}
\right.
\]
satisfying%
\begin{multline*}
X(x,t)=e^{t\left(  \boldsymbol{J}-\lambda\right)  }X_{0}\left(  x,0\right) \\
+%
%TCIMACRO{\dint \nolimits_{0}^{t}}%
%BeginExpansion
{\displaystyle\int\nolimits_{0}^{t}}
%EndExpansion
e^{\left(  t-s\right)  \left(  \boldsymbol{J}-\lambda\right)  }\left\{  \text{
}%
%TCIMACRO{\dint \limits_{\mathbb{Z}_{p}}}%
%BeginExpansion
{\displaystyle\int\limits_{\mathbb{Z}_{p}}}
%EndExpansion
A(x,y)f(X\left(  x,s+\theta\right)  )dy+%
%TCIMACRO{\dint \limits_{\mathbb{Z}_{p}}}%
%BeginExpansion
{\displaystyle\int\limits_{\mathbb{Z}_{p}}}
%EndExpansion
B(x,y)U(y)dy+Z(x)\right\}  ds\text{, }%
\end{multline*}
for $x\in\mathbb{Z}_{p}$, $t\in\left[  0,\tau\right]  ,\theta\in\left[
-r,0\right]  .$
\end{corollary}

\subsection{Stability of CNNs with delay}

\begin{theorem}
\label{Theorem_3}If $\lambda>\left\Vert J\right\Vert _{1}=1$, all the states
$X(x,t)$ of a $\text{CNN}(\lambda,J,A,B,U,Z)$ are bounded for all time
$t\geq0$. More precisely, if
\[
X_{\max}:=\left\Vert A\right\Vert _{\infty}\left\Vert f\right\Vert _{\infty
}+\left\Vert B\right\Vert _{\infty}\left\Vert U\right\Vert _{\infty
}+\left\Vert Z\right\Vert _{\infty},
\]
then
\begin{equation}
\sup_{t\geq0}\sup_{x\in\mathbb{Z}_{p}}|X(x,t)|\leq\frac{X_{\max}}%
{\lambda-\left\Vert J\right\Vert _{1}}. \label{No_Blow_up}%
\end{equation}
If $A(x,y),B(x,y)\in\mathcal{D}^{l}\left(  \mathbb{Z}_{p}\times\mathbb{Z}%
_{p},\mathbb{R}\right)  $, $J(x)$, $U(x)$, $Z(x)\in\mathcal{D}^{l}\left(
\mathbb{Z}_{p},\mathbb{R}\right)  $,%
\[
X_{\max}=\left\Vert f\right\Vert _{\infty}\max_{i,k\in G_{l}}\left\vert
A\left(  i,k\right)  \right\vert +\max_{i,k\in G_{l}}\left\vert B\left(
i,k\right)  \right\vert \text{ }\max_{i\in G_{l}}\left\vert U\left(  i\right)
\right\vert +\max_{i\in G_{l}}\left\vert Z\left(  i\right)  \right\vert .
\]

\end{theorem}

\begin{proof}
Set $X_{\max}:=\left\Vert A\right\Vert _{\infty}\left\Vert f\right\Vert
_{\infty}+\left\Vert B\right\Vert _{\infty}\left\Vert U\right\Vert _{\infty
}+\left\Vert Z\right\Vert _{\infty}$, then, the result follows from Corollary
\ref{Corollary_1} by using
\begin{align*}
\left\vert X(x,t)\right\vert  &  \leq e^{t\left(  \left\Vert J\right\Vert
_{1}-\lambda\right)  }\left\Vert X_{0}\right\Vert _{\infty}+X_{\max}%
%TCIMACRO{\dint \nolimits_{0}^{t}}%
%BeginExpansion
{\displaystyle\int\nolimits_{0}^{t}}
%EndExpansion
e^{\left(  t-s\right)  \left(  \left\Vert J\right\Vert _{1}-\lambda\right)
}\\
&  =e^{-t\left(  \lambda-\left\Vert J\right\Vert _{1}\right)  }\left\Vert
X_{0}\right\Vert _{\infty}+X_{\max}\left\{
\begin{array}
[c]{lll}%
t & \text{if} & \lambda=\left\Vert J\right\Vert _{1}\\
&  & \\
\frac{1-e^{-\left(  \lambda-\left\Vert J\right\Vert _{1}\right)  t}}%
{\lambda-\left\Vert J\right\Vert _{1}} & \text{if} & \lambda>\left\Vert
J\right\Vert _{1}\\
&  & \\
\frac{\left(  e^{\left(  \left\Vert J\right\Vert _{1}-\lambda\right)
t}-1\right)  }{\left\Vert J\right\Vert _{1}-\lambda} & \text{if} &
\lambda<\left\Vert J\right\Vert _{1}.
\end{array}
\right.
\end{align*}

\end{proof}

\subsection{Some numerical examples}

In this section, we present some numerical approximations for the solution of
the Cauchy problem (\ref{Cauchy_1}). We use the numerical techniques developed
in (\cite{Zambrano-Zuniga-1})-(\cite{Zambrano-Zuniga-2}). The time step used
$\delta_{t}=0.05$, $\ $and the time runs through the ninterval $\left[
0,30\right]  $. We use $p=2$, and use truncated 2-adic numbers with five
digits, $l=5$; these numbers correspond to the vertices at the top level of
the tree $G_{5}$. We also use%
\[
J(x)=%
\begin{cases}
2^{2} & \text{if }\left\vert x\right\vert _{2}\leq2^{-2}\\
& \\
0 & \text{otherwise,}%
\end{cases}
\text{ }%
\]%
\begin{align*}
U(x)  &  =\sin(\pi(1-\left\vert x\right\vert _{2}))\text{, }A(x)=-4\sin
(\pi(1-\left\vert x\right\vert _{2}))\text{,}\\
B(x)  &  =\cos(\pi(1-\left\vert x\right\vert _{2}))\text{, }Z(x)=-0.15\Omega
\left(  \left\vert x\right\vert _{2}\right)  \text{, }%
\end{align*}

\[
\phi(x,s)=%
\begin{cases}
\cos(\pi\left\vert x\right\vert _{2})\cdot\sin(\pi\cdot s) & \text{if
}\left\vert x\right\vert _{2}\leq2^{-3}\\
& \\
\exp(-\left\vert x\right\vert _{2})\cdot\sin(\pi\cdot s) & \text{otherwise,}%
\end{cases}
\]
where $x\in\mathbb{Z}_{p}$, and $s\in\lbrack-4,0]$, thus $r=-4$. We use
$\lambda=0.5$, $2.0$. The numerical experiments show that in $p$-adic CNNs
without delay the steady state response of the network does not depent on the
initial datum $\phi(x,0)$; while in case with delay, the steady state response
depend heavily on the `memory $\phi(x,s).$' The memory induces a `chaotic
type' steady state response in the network.%

%TCIMACRO{\FRAME{ftbpFU}{2.744in}{2.7942in}{0pt}{\Qcb{The top figure is the
%response of the network with no delay, i.e. $r=0$. The bottom figure is the
%response of the network wit $r=-4$. In both cases $\lambda=0.5.$}}{\Qlb{Figure
%7}}{Figure 7}{\special{ language "Scientific Word";  type "GRAPHIC";
%maintain-aspect-ratio TRUE;  display "USEDEF";  valid_file "T";
%width 2.744in;  height 2.7942in;  depth 0pt;  original-width 6.6392in;
%original-height 6.7637in;  cropleft "0";  croptop "1";  cropright "1";
%cropbottom "0";  tempfilename 'SEEU6N05.wmf';tempfile-properties "XPR";}} }%
%BeginExpansion
\begin{figure}
[h]
\begin{center}
\includegraphics[width=0.75\textwidth]%
{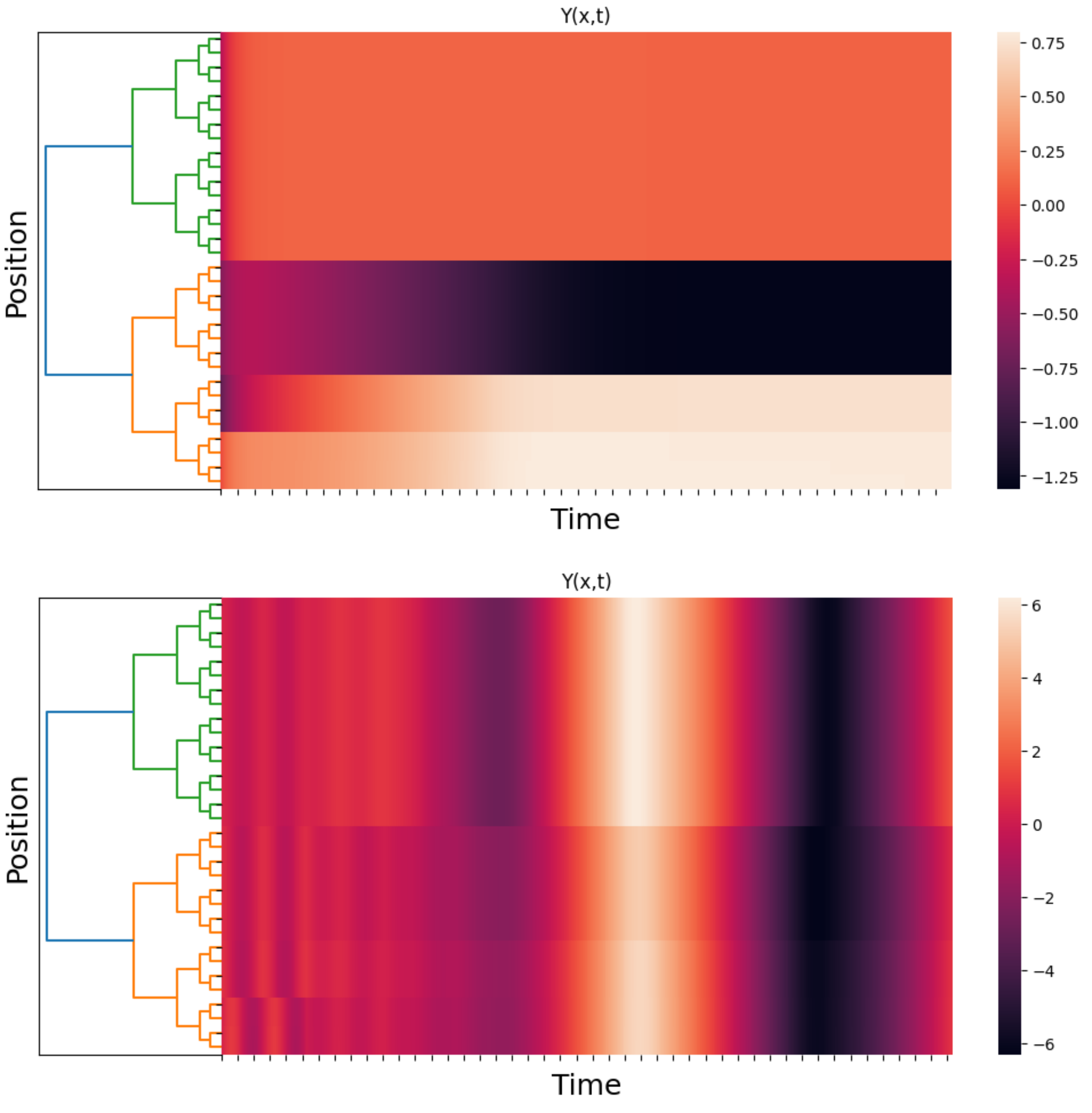}%
\caption{The top figure is the response of the network with no delay, i.e.
$r=0$. The bottom figure is the response of the network wit $r=-4$. In both
cases $\lambda=0.5.$}%
\label{Figure 7}%
\end{center}
\end{figure}
%EndExpansion
%TCIMACRO{\FRAME{ftbpFU}{2.7691in}{2.7942in}{0pt}{\Qcb{The top figure is the
%response of the network with no delay, i.e. $r=0$. The bottom figure is the
%response of the networks wit $r=-4$. In both cases $\lambda=2.0$.}%
%}{\Qlb{Figure 8}}{Figure 8}{\special{ language "Scientific Word";
%type "GRAPHIC";  maintain-aspect-ratio TRUE;  display "USEDEF";
%valid_file "T";  width 2.7691in;  height 2.7942in;  depth 0pt;
%original-width 6.7014in;  original-height 6.7637in;  cropleft "0";
%tempfilename 'SEEU9N06.wmf';tempfile-properties "XPR";}} }%
%BeginExpansion
\begin{figure}
[h]
\begin{center}
\includegraphics[width=0.75\textwidth]%
{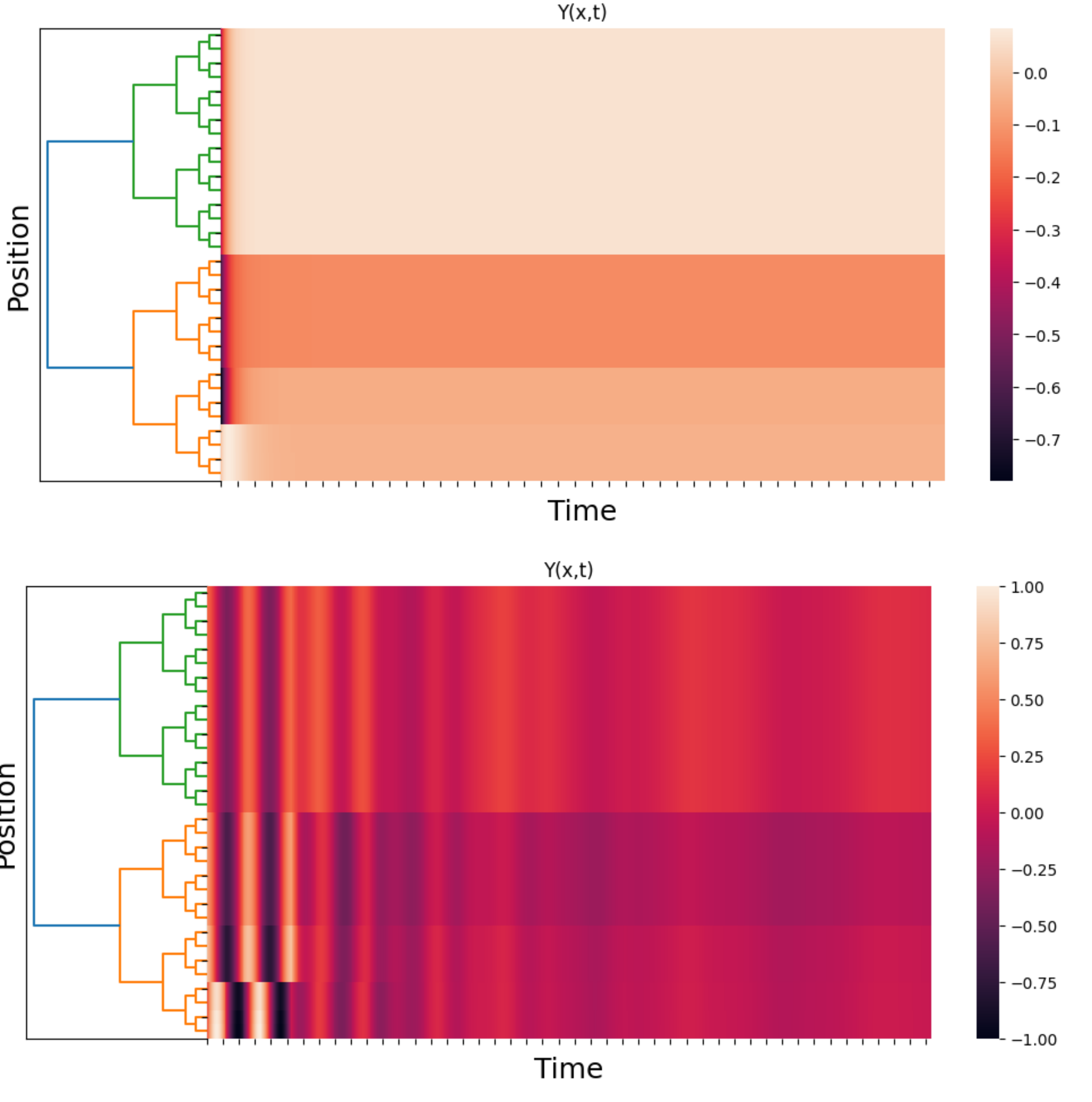}%
\caption{The top figure is the response of the network with no delay, i.e.
$r=0$. The bottom figure is the response of the networks wit $r=-4$. In both
cases $\lambda=2.0$.}%
\label{Figure 8}%
\end{center}
\end{figure}
%EndExpansion

\section{Applications of $p$-adic CNNs in image processing: edge detection}

In this section, we review the edge detectors based on $p$-adic CNNs for
grayscale images introduced in \cite{Zambrano-Zuniga-2}. We take $B\in
L^{1}(\mathbb{Z}_{p})$ and $U,Z\in\mathcal{C}(\mathbb{Z}_{p})$, $a$,
$b\in\mathbb{R}$ , and fix the sigmoidal function $f(s)=\frac{1}{2}(\left\vert
s+1\right\vert -|s-1|)$ for $s\in\mathbb{R}$. In this section, we consider the
following $p$-adic CNN:%
\begin{equation}
\left\{
\begin{array}
[c]{ll}%
\frac{\partial}{\partial t}X(x,t)=-X(x,t)+aY(x,t)+(B\ast U)(x)+Z(x), &
x\in\mathbb{Z}_{p},t\geq0;\\
& \\
Y(x,t)=f(X(x,t)). &
\end{array}
\right.  \label{CNN_1_A}%
\end{equation}
We denote this $p$-adic CNN as $CNN(a,B,U,Z)$, where $a,B,U,Z$\ are the
parameters of the network. In applications to edge detection, we take $U(x)$
to be a grayscale image, and take the initial datum as $X(x,0)=0$. The
simulations show that after a time sufficiently large the network outputs a
white and black image approximating the edges of the original image $U(x)$.
The performance of this edge detector is comparable to the Canny detector, and
other well-know detectors. But most importantly, they can explain, reasonably
well, how the network learns the edges of an image.

\subsection{Stationary states}

We say that $X_{stat}(x)$ is a \textit{stationary state} of the network
$CNN(a,B,U,Z)$, if
\begin{equation}
\left\{
\begin{array}
[c]{ll}%
X_{stat}(x)=aY_{stat}(x)+(B\ast U)(x)+Z(x), & x\in\mathbb{Z}_{p};\\
& \\
Y_{stat}(x)=f(X_{stat}(x)). &
\end{array}
\right.  \label{Sationary state}%
\end{equation}

\begin{lemma}
\cite[Lemma 1]{Zambrano-Zuniga-2}\label{Lemma-0}(i) If $a<1$, then the network
$CNN(a,B,U,Z)$ has a unique stationary state $X_{stat}(x)\in\mathcal{C}%
(\mathbb{Z}_{p})$ given by%
\begin{equation}
X_{stat}(x)=\left\{
\begin{array}
[c]{lcr}%
a+(B\ast U)(x)+Z(x) & \text{if} & (B\ast U)(x)+Z(x)>1-a\\
-a+(B\ast U)(x)+Z(x) & \text{if} & (B\ast U)(x)+Z(x)<-1+a\\
\frac{(B\ast U)(x)+Z(x)}{1-a} & \text{if} & |(B\ast U)(x)+Z(x)|\leq1-a.
\end{array}
\right.  \label{Case_1}%
\end{equation}

\noindent(ii) If $a=1$ , then the network $CNN(a,B,U,Z)$ has a unique
stationary state $X_{stat}(x)\in L^{1}(\mathbb{Z}_{p})$ given by
\begin{equation}
X_{stat}(x)=\left\{
\begin{array}
[c]{lcr}%
1+(B\ast U)(x)+Z(x) & \text{if} & (B\ast U)(x)+Z(x)>0\\
-1+(B\ast U)(x)+Z(x) & \text{if} & (B\ast U)(x)+Z(x)<0\\
0 & \text{if} & (B\ast U)(x)+Z(x)=0.
\end{array}
\right.  \label{Case_2}%
\end{equation}

\end{lemma}

\begin{definition}
\label{Definition1}Assume that $a>1$. Given
\[
I_{+}\subseteq\{x\in\mathbb{Z}_{p};\;1-a<(B\ast U)(x)+Z(x)\},
\]%
\[
I_{-}\subseteq\{x\in\mathbb{Z}_{p};\;(B\ast U)(x)+Z(x)<a-1\},
\]
satisfying $I_{+}\cap I_{-}=\varnothing$ \ and
\[
\mathbb{Z}_{p}\smallsetminus\left(  I_{+}\cup I_{-}\right)  \subseteq
\{x\in\mathbb{Z}_{p};\;1-a<(B\ast U)(x)+Z(x)<a-1\},
\]
we define the function%
\begin{equation}
X_{stat}(x;I_{+},I_{-})=\left\{
\begin{array}
[c]{lll}%
a+(B\ast U)(x)+Z(x) & \text{if} & x\in I_{+}\\
-a+(B\ast U)(x)+Z(x) & \text{if} & x\in I_{-}\\
\frac{(B\ast U)(x)+Z(x)}{1-a} & \text{if} & x\in\mathbb{Z}_{p}\setminus\left(
I_{+}\cup I_{-}\right)  .
\end{array}
\right.  \label{Stationary_Sol_3}%
\end{equation}

\end{definition}

\begin{theorem}
\cite[Theorem 1]{Zambrano-Zuniga-2}\label{Theorem1}Assume that $a>1$. All
functions of type (\ref{Stationary_Sol_3}) are stationary states of the
network $CNN(a,B,U,Z)$. Conversely, any \ stationary state of the network
$CNN(a,B,U,Z)$ has the form (\ref{Stationary_Sol_3}).
\end{theorem}

\begin{remark}
Notice that%
\[
Y_{stat}(x;I_{+},I_{-}):=f\left(  X_{stat}(x;I_{+},I_{-})\right)  =\left\{
\begin{array}
[c]{lll}%
1 & \text{if} & x\in I_{+}\\
-1 & \text{if} & x\in I_{-}\\
\frac{(B\ast U)(x)+Z(x)}{1-a} & \text{if} & x\in\mathbb{Z}_{p}\setminus\left(
I_{+}\cup I_{-}\right)  .
\end{array}
\right.
\]
The function $Y_{stat}(x;I_{+},I_{-})$ is the output of the network. If
$I_{+}\cup I_{-}=\mathbb{Z}_{p}$, we say that $X_{stat}(x;I_{+},I_{-})$ is
bistable. The set $\mathcal{B}\left(  I_{+},I_{-}\right)  =\mathbb{Z}%
_{p}\setminus\left(  I_{+}\cup I_{-}\right)  $ measures how far $X_{stat}%
(x;I_{+},I_{-})$ is from being bistable. We call set $\mathcal{B}\left(
I_{+},I_{-}\right)  $ the set of bistability of $X_{stat}(x;I_{+},I_{-})$. If
$\mathcal{B}\left(  I_{+},I_{-}\right)  =\varnothing$, then $X_{stat}%
(x;I_{+},I_{-})$ is bistable.
\end{remark}

\begin{remark}
If $I_{+}\cup I_{-}\subsetneqq\mathbb{Z}_{p}$, we say that $X_{stat}%
(x;I_{+},I_{-})$ is an unstable.
\end{remark}

\subsection{Hierarchical structure of the space of stationary states}

A relation $\preccurlyeq$ is \textit{a partial order} on a set $S$ if it
satisfies: 1 (reflexivity) $f\preccurlyeq f$ for all $f$ in $S$; 2
(antisymmetry) $f\preccurlyeq g$ and $g\preccurlyeq f$ implies $f=g$; 3
(transitivity) $f\preccurlyeq g$ and $g\preccurlyeq h$ implies $f\preccurlyeq
h$. \ A \textit{partially ordered set} $\left(  S,\preccurlyeq\right)  $ (or
poset) is a set endowed with a partial order. A partially ordered set $\left(
S,\preccurlyeq\right)  $ is called a \textit{lattice} if for every $f$, $g$ in
$S$, the elements $f\wedge g=\inf\{f,g\}$ and $f\vee$ $g=\sup\{f,g\}$ exist.
Here, $f\wedge g$ denotes the smallest element in $S$ satisfying $f\wedge
g\preccurlyeq f$ and $f\wedge g\preccurlyeq g$; while $f\vee$ $g$ \ denotes
the largest element in $S$ satisfying $f\preccurlyeq$ $f\vee$ $g$ and
$g\preccurlyeq f\vee$ $g$. We say that $h\in S$ a \textit{minimal} element of
with respect to $\preccurlyeq$, if there is no element $f\in S$, $f\neq h$
such that $f\preccurlyeq h$.

Posets offer a natural way to formalize the notion of hierarchy.

We set
\[
\mathcal{M}=\bigcup_{I_{+},I_{-}}\left\{  X_{stat}(x;I_{+},I_{-})\right\}  ,
\]
where $I_{+},I_{-}$ run trough all the sets given in Definition
\ref{Definition1}. \ Given $X_{stat}(x;I_{+},I_{-})$ and $X_{stat}%
(x;I_{+}^{\prime},I_{-}^{\prime})$ in $\mathcal{M}$, with $I_{+}\cup I_{-}%
\neq\mathbb{Z}_{p}$ or $I_{+}^{\prime}\cup I_{-}^{\prime}\neq\mathbb{Z}_{p}$,
we define%
\begin{equation}
X_{stat}(x;I_{+}^{\prime},I_{-}^{\prime})\preccurlyeq X_{stat}(x;I_{+}%
,I_{-})\text{ if }I_{+}\cup I_{-}\subseteq I_{+}^{\prime}\cup I_{-}^{\prime}.
\label{Definitioon_Order}%
\end{equation}
In the case $I_{+}\cup I_{-}=\mathbb{Z}_{p}$ and $I_{+}^{\prime}\cup
I_{-}^{\prime}=\mathbb{Z}_{p}$, the corresponding stationary states
$X_{stat}(x;I_{+},I_{-})$, $X_{stat}(x;I_{+},I_{-})$ are not comparable. In
\cite{Zambrano-Zuniga-2}, the authors show that (\ref{Definitioon_Order})
defines a partial order in $\mathcal{M}$.

\begin{theorem}
\cite[Theorem 2]{Zambrano-Zuniga-2}\label{Theorem2} $\left(  \mathcal{M}%
,\preccurlyeq\right)  $ is a lattice. Furthermore, the set of minimal elements
of $\left(  \mathcal{M},\preccurlyeq\right)  $ agrees with the set of bistable
states of $CNN(a,B,U,Z)$.
\end{theorem}

\subsection{A new class of edge detectors}

In \cite[Theorem 2]{Zambrano-Zuniga-2}, the authors implemented a numerical
method for solving the initial value problem attached to network
$CNN(a,B,U,Z)$, with $X(x,0)=0$ and $U(x)$ a grayscale image. The simulations
show that after a sufficiently large time the network outputs a
black-and-white image approximating the edges of the original image $U(x)$.
This means that for $t$ sufficiently large $X(x,t)$ is close to a bistable
stationary state $X_{stat}(x;I_{+},I_{-})$. Furthermore, after a certain
sufficiently large time, the output of the network do not show a difference
perceivable by the human eye. See Figures \ref{Figure 9}-\ref{Figure 10}. The
performance of this edge detector is comparable to the Canny detector, and
other well-known detectors. But most importantly, we can explain, reasonably
well, how the network detects the edges of an image.

An intuitive picture of the dynamics of the network is as follows. For $t$
sufficiently large, the network performs transitions between stationary states
$X_{stat}(x;I_{+},I_{-})$ belonging to a small neighborhood $\mathcal{N}$
around a bistable state $X_{stat}^{\left(  0\right)  }(x;I_{+},I_{-})$, with
$I_{+}\cup I_{-}=\mathbb{Z}_{p}$. The dynamics of the network consists of
transitions in a hierarchically organized landscape\textit{ }$\left(
\mathcal{M},\preccurlyeq\right)  $ toward some minimal state. This is a
reformulation of the classical paradigm asserting that the dynamics of a large
class of complex systems can be modeled as a random walk on its energy landscape.%

%TCIMACRO{\FRAME{ftbpFU}{4.6588in}{2.8124in}{0pt}{\Qcb{Left side, the original
%image. Right side, edges obtained by using a Canny edge detector. Taken from
%\cite{Zambrano-Zuniga-2}.}}{\Qlb{Figure 9}}{Figure 9}%
%{\special{ language "Scientific Word";  type "GRAPHIC";  display "USEDEF";
%valid_file "T";  width 4.6588in;  height 2.8124in;  depth 0pt;
%original-width 6.3399in;  original-height 3.1531in;  cropleft "0";
%tempfilename 'SEIC7H00.wmf';tempfile-properties "XPR";}} }%
%BeginExpansion
\begin{figure}
[h]
\begin{center}
\includegraphics[width=0.75\textwidth]%
{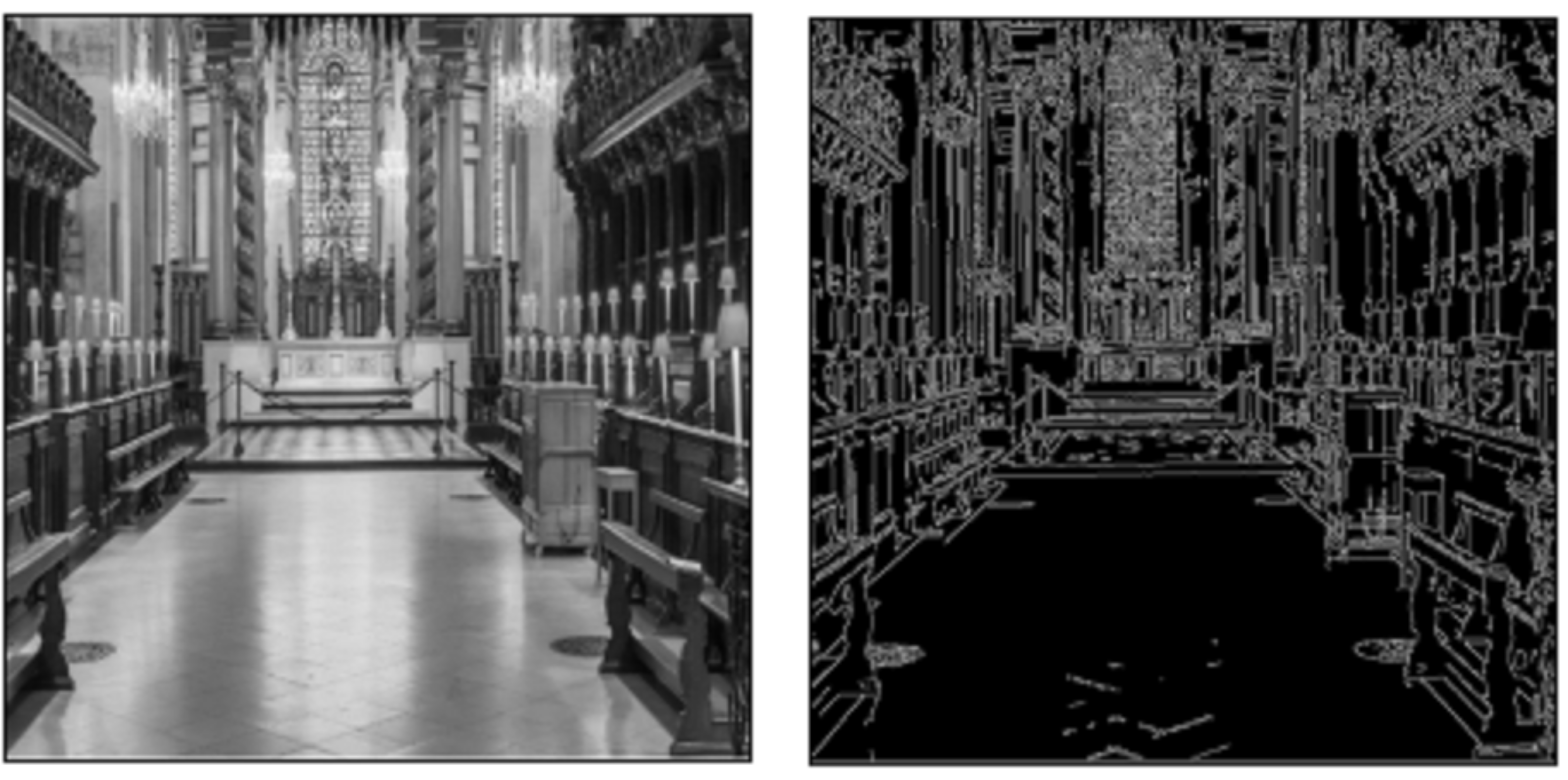}%
\caption{Left side, the original image. Right side, edges obtained by using a
Canny edge detector. Taken from \cite{Zambrano-Zuniga-2}.}%
\label{Figure 9}%
\end{center}
\end{figure}
%EndExpansion

\bigskip%
%TCIMACRO{\FRAME{ftbpFU}{4.7132in}{2.5477in}{0pt}{\Qcb{On the left side, edges
%were obtained by using a discretization of the CNN (\ref{CNN_1_A}), with
%$z_{0}=-1$ and $6$ steps. On the right side, edges were obtained by using a
%discretization of the CNN (\ref{CNN_1_A}), with $z_{0}=-1$ and $10$ steps. See
%\cite{Zambrano-Zuniga-2} for further details.}}{\Qlb{Figure 10}}{Figure
%10}{\special{ language "Scientific Word";  type "GRAPHIC";
%maintain-aspect-ratio TRUE;  display "USEDEF";  valid_file "T";
%width 4.7132in;  height 2.5477in;  depth 0pt;  original-width 6.4792in;
%original-height 3.4861in;  cropleft "0";  croptop "1";  cropright "1";
%cropbottom "0";  tempfilename 'SEICCY01.wmf';tempfile-properties "XPR";}} }%
%BeginExpansion
\begin{figure}
[h]
\begin{center}
\includegraphics[width=0.75\textwidth]%
{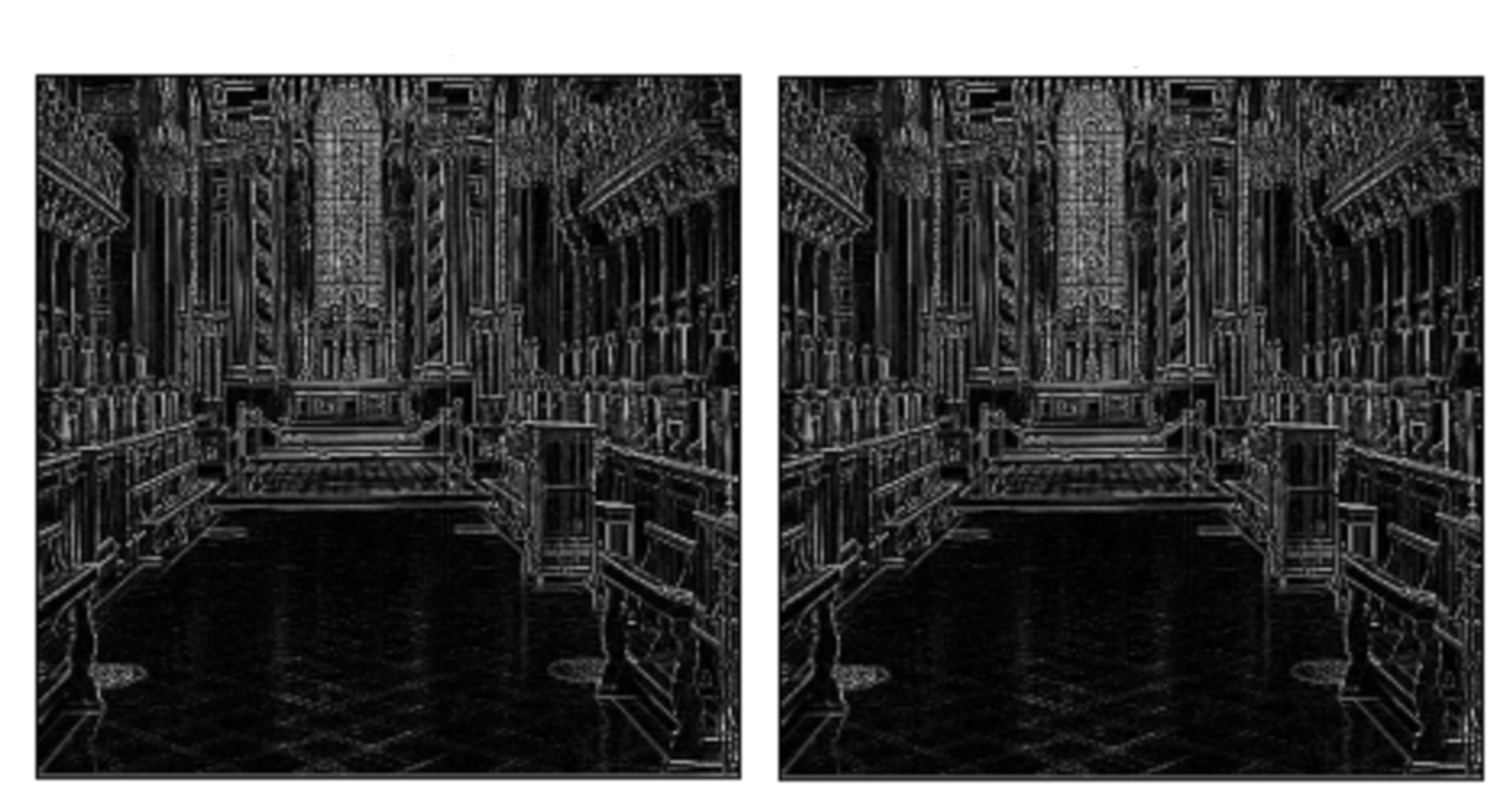}%
\caption{On the left side, edges were obtained by using a discretization of
the CNN (\ref{CNN_1_A}), with $z_{0}=-1$ and $6$ steps. On the right side,
edges were obtained by using a discretization of the CNN (\ref{CNN_1_A}), with
$z_{0}=-1$ and $10$ steps. See \cite{Zambrano-Zuniga-2} for further details.}%
\label{Figure 10}%
\end{center}
\end{figure}
%EndExpansion

\section{Applications of $p$-adic CNNs in image processing: edge detection:
denoising}

In this section, we review the filtering techniques introduced in
\cite{Zambrano-Zuniga-2} for grayscale images polluted with Gaussian noise.

\subsection{The $p$-adic heat equation}

The field of $p$-adic numbers $\mathbb{Q}_{p}$ is the quotient field of the
ring $\mathbb{Z}_{p}$:%
\[
\mathbb{Q}_{p}=\left\{  \frac{a}{b};a,b\in\mathbb{Z}_{p}\text{, with }%
b\neq0\right\}  .
\]
Any $p-$adic number $x\neq0$ has a unique expansion of the form
\begin{equation}
x=p^{ord_{p}(x)}\sum_{j=0}^{\infty}x_{j}p^{j}, \label{expansion}%
\end{equation}
where $x_{j}\in\{0,1,2,\dots,p-1\}$, $x_{0}\neq0$, and $ord_{p}(x)=ord(x)\in
\mathbb{Z}$. It follows from (\ref{expansion}), that any $x\in\mathbb{Q}%
_{p}\smallsetminus\left\{  0\right\}  $ can be represented uniquely as
$x=p^{ord(x)}u\left(  x\right)  $ and $\left\vert x\right\vert _{p}%
=p^{-ord(x)}$.

We denote by $\mathcal{D}\left(  \mathbb{Q}_{p}\right)  $ the $\mathbb{R}%
$-vector space\ of test functions defined on $\mathbb{Q}_{p}$, $L^{2}%
(\mathbb{Q}_{p})$ the $\mathbb{R}$-vector space\ of square integrable
functions defined on $\mathbb{Q}_{p}$, and by $C\left(  \mathbb{Q}_{p}\right)
$ the $\mathbb{R}$-vector space\ of continuous functions defined on
$\mathbb{Q}_{p}$.

For $\alpha>0$, the Vladimirov-Taibleson operator $\boldsymbol{D}^{\alpha}$ is
defined as
\[%
\begin{array}
[c]{ccc}%
\mathcal{D}(\mathbb{Q}_{p}) & \rightarrow & L^{2}(\mathbb{Q}_{p})\cap C\left(
\mathbb{Q}_{p}\right) \\
&  & \\
\varphi & \rightarrow & \boldsymbol{D}^{\alpha}\varphi,
\end{array}
\]
where%
\[
\left(  \boldsymbol{D}^{\alpha}\varphi\right)  \left(  x\right)
=\frac{1-p^{\alpha}}{1-p^{-\alpha-1}}%
%TCIMACRO{\dint \limits_{\mathbb{Q}_{p}}}%
%BeginExpansion
{\displaystyle\int\limits_{\mathbb{Q}_{p}}}
%EndExpansion
\frac{\left[  \varphi\left(  x-y\right)  -\varphi\left(  x\right)  \right]
}{\left\vert y\right\vert _{p}^{\alpha+1}}dy.
\]
The $p$-adic analogue of the heat equation is%
\[
\frac{\partial u\left(  x,t\right)  }{\partial t}+a\boldsymbol{D}^{\alpha
}u\left(  x,t\right)  =0\text{, with }a>0\text{.}%
\]
The solution of the Cauchy problem attached to the heat equation with initial
datum $u\left(  x,0\right)  =\varphi\left(  x\right)  \in\mathcal{D}%
(\mathbb{Q}_{p})$ is given by%
\[
u\left(  x,t\right)  =%
%TCIMACRO{\dint \limits_{\mathbb{Q}_{p}}}%
%BeginExpansion
{\displaystyle\int\limits_{\mathbb{Q}_{p}}}
%EndExpansion
Z\left(  x-y,t\right)  \varphi\left(  y\right)  dy,
\]
where $Z\left(  x,t\right)  $ is the $p$\textit{-adic heat kernel} defined as
\begin{equation}
Z\left(  x,t\right)  =%
%TCIMACRO{\dint \limits_{\mathbb{Q}_{p}}}%
%BeginExpansion
{\displaystyle\int\limits_{\mathbb{Q}_{p}}}
%EndExpansion
\chi_{p}\left(  -x\xi\right)  e^{-at\left\vert \xi\right\vert _{p}^{\alpha}%
}d\xi, \label{het-kernel}%
\end{equation}
where $\chi_{p}\left(  -x\xi\right)  $\ is the standard additive character of
the group $\left(  \mathbb{Q}_{p},+\right)  $. The $p$-adic heat kernel is the
transition density function of a Markov stochastic process with space state
$\mathbb{Q}_{p}$, see, e.g., \cite{Kochubei}-\cite{Zuniga-LNM-2016}.

\subsection{The $p$-adic heat equation on the unit ball}

We define the operator $\boldsymbol{D}_{0}^{\alpha}$, $\alpha>0$, by
restricting $\boldsymbol{D}^{\alpha}$ to $\mathcal{D}(\mathbb{Z}_{p})$ and
considering $\left(  \boldsymbol{D}^{\alpha}\varphi\right)  \left(  x\right)
$ only for $x\in\mathbb{Z}_{p}$. The operator $\boldsymbol{D}_{0}^{\alpha}%
$\ satisfies
\[
\boldsymbol{D}_{0}^{\alpha}\varphi(x)=\lambda\varphi(x)+\frac{1-p^{\alpha}%
}{1-p^{-\alpha-1}}%
%TCIMACRO{\dint \limits_{\mathbb{Z}_{p}}}%
%BeginExpansion
{\displaystyle\int\limits_{\mathbb{Z}_{p}}}
%EndExpansion
\frac{\varphi(x-y)-\varphi(x)}{|y|_{p}^{\alpha+1}}dy,
\]
for $\mathbb{\varphi\in}\mathcal{D}(\mathbb{Z}_{p})$, with $\lambda=\frac
{p-1}{p^{\alpha+1}-1}p^{\alpha}$.

Consider the Cauchy problem%
\[
\left\{
\begin{array}
[c]{lll}%
\frac{\partial u\left(  x,t\right)  }{\partial t}+\boldsymbol{D}_{0}^{\alpha
}u\left(  x,t\right)  -\lambda u\left(  x,t\right)  =0\text{, } &
x\in\mathbb{Z}_{p}, & t>0;\\
&  & \\
u\left(  x,0\right)  =\varphi\left(  x\right)  , & x\in\mathbb{Z}_{p}, &
\end{array}
\right.
\]
where $\mathbb{\varphi\in}\mathcal{D}(\mathbb{Z}_{p})$. The solution of this
problem is given by%
\[
u\left(  x,t\right)  =%
%TCIMACRO{\dint \limits_{\mathbb{Z}_{p}}}%
%BeginExpansion
{\displaystyle\int\limits_{\mathbb{Z}_{p}}}
%EndExpansion
Z_{0}(x-y,t)\varphi\left(  y\right)  dy\text{, }x\in\mathbb{Z}_{p}\text{,
}t>0,
\]
where
\begin{align*}
Z_{0}(x,t)  &  :=e^{\lambda t}Z(x,t)+c(t)\text{, }x\in\mathbb{Z}_{p}\text{,
}\\
c(t)  &  :=1-(1-p^{-1})e^{\lambda t}\sum_{n=0}^{\infty}\frac{(-1)^{n}}%
{n!}t^{n}\frac{1}{1-p^{-n\alpha-1}}%
\end{align*}
and $Z(x,t)$ is given (\ref{het-kernel}). The function $Z_{0}(x,t)$ is
non-negative for $x\in\mathbb{Z}_{p}$, $t>0$, and
\[%
%TCIMACRO{\dint \limits_{\mathbb{Z}_{p}}}%
%BeginExpansion
{\displaystyle\int\limits_{\mathbb{Z}_{p}}}
%EndExpansion
Z_{0}(x,t)dx=1,
\]
\cite{Kochubei}. Furthermore, $Z_{0}(x,t)$ is the transition density function
of a Markov process with space state $\mathbb{Z}_{p}$.

The family
\begin{equation}%
\begin{array}
[c]{cccc}%
T_{t}: & L^{1}(\mathbb{Z}_{p}) & \rightarrow & L^{1}(\mathbb{Z}_{p})\\
&  &  & \\
& \phi(x) & \rightarrow & T_{t}\phi(x):=%
%TCIMACRO{\dint \limits_{\mathbb{Z}_{p}}}%
%BeginExpansion
{\displaystyle\int\limits_{\mathbb{Z}_{p}}}
%EndExpansion
Z_{0}(x-y,t)\phi(y)dy,
\end{array}
\label{Khrennikov-Kochubei}%
\end{equation}
is a $C^{0}$-semigroup of contractions with generator $\boldsymbol{D}%
_{0}^{\alpha}-\lambda I$ on $L^{1}(\mathbb{Z}_{p})$, see \cite[Proposition 4,
Proposition 5]{Khrennikov-Kochubei}

\subsection{Reaction-diffusion CNNs}

Given $\mu\in\mathbb{R}$, $\alpha>0$, $A$, $B$,$U$, $Z\in\mathcal{C}%
(\mathbb{Z}_{p})$, a $p$-adic reaction-diffusion CNN, denoted as $CNN\left(
\mu,\alpha,A,B,U,Z\right)  $, is the dynamical system given by the following
integro-differential equation:%
\begin{align}
\frac{\partial X(x,t)}{\partial t}  &  =\mu X(x,t)+(\lambda I-\boldsymbol{D}%
_{0}^{\alpha})X(x,t)+%
%TCIMACRO{\dint \limits_{\mathbb{Z}_{p}}}%
%BeginExpansion
{\displaystyle\int\limits_{\mathbb{Z}_{p}}}
%EndExpansion
A(x-y)f(X(y,t))dy\label{RD-CNN}\\
&  +%
%TCIMACRO{\dint \limits_{\mathbb{Z}_{p}}}%
%BeginExpansion
{\displaystyle\int\limits_{\mathbb{Z}_{p}}}
%EndExpansion
B(x-y)U(y)dy+Z(x),\nonumber
\end{align}
where $x\in\mathbb{Z}_{p}$, $t\geq0$. We say that $X(x,t)\in$ $\mathbb{R}$ is
the state of cell $x$ at the time $t$. Function $A$ is the kernel of the
feedback operator, while function $B$ is the kernel of the feedforward
operator. Function $U$ is the input of the CNN, while function $Z$ is the
threshold of the CNN.

Notice that if $\mu=0$ and $A=B=U=Z=0$, (\ref{RD-CNN}) becomes the $p$-adic
heat equation in the unit ball. Then, in (\ref{RD-CNN}), $(\lambda
I-\boldsymbol{D}_{0}^{\alpha})$ is\ the diffusion term, while the other terms
are the reaction ones, which describe the interaction between $X(x,t)$,
$U(x)$, and $Z(x)$.

The following result establishes the existence of uniqueness of a mild
solution for the Cauchy problem associated with (\ref{RD-CNN}). This result is
sufficient to study the stability of these networks.

\begin{proposition}
\cite[Proposition 1, Lemma 3]{Zambrano-Zuniga-2}(i) Let $A$, $B$, $U$,
$Z\in\mathcal{C}(\mathbb{Z}_{p})$. Take $X_{0}\in L^{1}(\mathbb{Z}_{p})$ as
the initial datum for the Cauchy problem attached to (\ref{RD-CNN}). Then
there exists $\tau=\tau\left(  X_{0}\right)  \in\left(  0,\infty\right]  $ and
a unique $X(t)\in C([0,\tau],L^{1}(\mathbb{Z}_{p}))$ satisfying
\begin{equation}
\left\{
\begin{array}
[c]{l}%
X(t)=e^{\mu t}T_{t}X_{0}+\int_{0}^{t}e^{\mu(t-s)}T_{t-s}\boldsymbol{H}%
(X(s))ds\\
X(0)=X_{0}.
\end{array}
\right.  \label{Eq:mild solution}%
\end{equation}

(ii) Let $A$, $B$, $U$, $Z\in\mathcal{C}(\mathbb{Z}_{p})$. Take $X_{0}%
\in\mathcal{C}(\mathbb{Z}_{p})$. Then, the integral equation
(\ref{Eq:mild solution}) has unique solution $C([0,\infty),C(\mathbb{Z}_{p}))$.
\end{proposition}

\begin{theorem}
\cite[Theorem 3]{Zambrano-Zuniga-2}Let $X(t)\in C([0,\infty),\mathcal{C}%
(\mathbb{Z}_{p}))$ be the unique solution of (\ref{Eq:mild solution}), with
initial condition $X_{0}\in C(\mathbb{Z}_{p})$.Then,
\begin{equation}
\Vert X(t)\Vert_{\infty}\leq e^{\mu t}\Vert X_{0}\Vert_{\infty}+\frac{(e^{\mu
t}-1)}{\mu}\left(  \Vert A\Vert_{1}\Vert f\Vert_{\infty}+\Vert B\Vert_{1}\Vert
U\Vert_{\infty}+\Vert Z\Vert_{\infty}\right)  , \label{Eq-1A}%
\end{equation}
if $\mu\neq0$, otherwise%
\begin{equation}
\Vert X(t)\Vert_{\infty}\leq\Vert X_{0}\Vert_{\infty}+\tau\left(  \Vert
A\Vert_{1}\Vert f\Vert_{\infty}+\Vert B\Vert_{1}\Vert U\Vert_{\infty}+\Vert
Z\Vert_{\infty}\right)  . \label{Eq-1B}%
\end{equation}

\end{theorem}

\subsection{Denoising}

In this section, we review the denoising technique based on reaction-diffusion
CNNs introduced in \cite{Zambrano-Zuniga-2}. We first consider the initial
value problem%
\begin{equation}
\left\{
\begin{array}
[c]{lll}%
\frac{\partial X(x,t)}{\partial t}+\boldsymbol{D}_{0}^{1}X(x,t)-\lambda
X(x,t)=0, & x\in\mathbb{Z}_{p}, & t>0\\
&  & \\
X(x,0)=X_{0}(x), & x\in\mathbb{Z}_{p}, &
\end{array}
\right.  \label{Eq: Vladimirov}%
\end{equation}
where $X_{0}(x)\in$ $[0,1]$ is a grayscale image codified as a test function
supported in the unit ball $\mathbb{Z}_{p}$. The algorithm for this coding is
discussed in \cite{Zambrano-Zuniga-2}. The output image $X(x,t)$ is similar to
the one produced by the classical Gaussian filter. See Figure \ref{Figure 11}.%

%TCIMACRO{\FRAME{ftbpFU}{4.6726in}{2.3566in}{0pt}{\Qcb{On the left side, the
%original image $X(x,0)$. On the right side $X(x,3)$. Taken from
%\cite{Zambrano-Zuniga-2}.}}{\Qlb{Figure 11}}{Figure 11}%
%maintain-aspect-ratio TRUE;  display "USEDEF";  valid_file "T";
%width 4.6726in;  height 2.3566in;  depth 0pt;  original-width 6.3192in;
%original-height 3.1739in;  cropleft "0";  croptop "1";  cropright "1";
%cropbottom "0";  tempfilename 'SEIENJ02.wmf';tempfile-properties "XPR";}} }%
%BeginExpansion
\begin{figure}
[h]
\begin{center}
\includegraphics[width=0.75\textwidth]%
{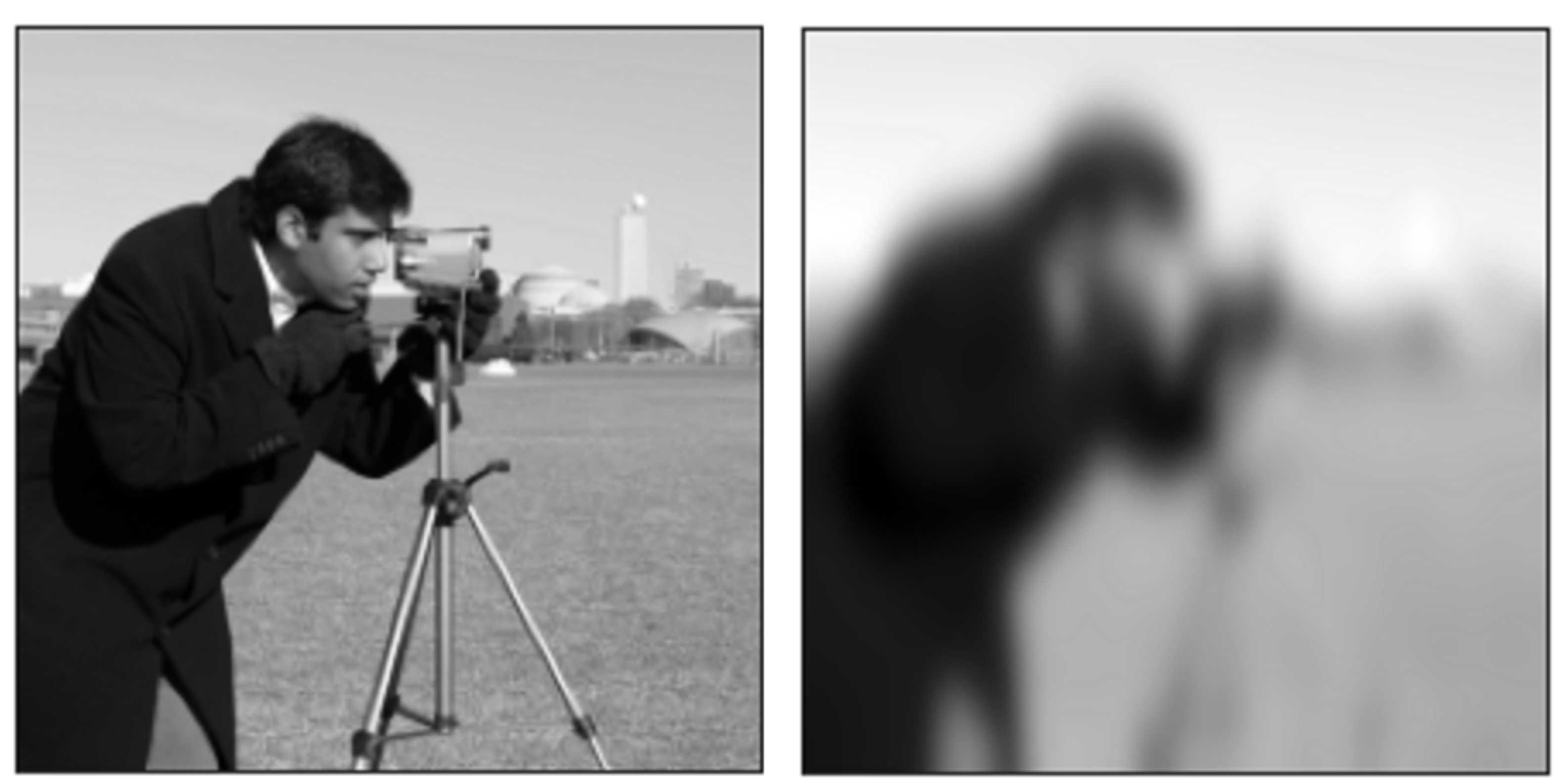}%
\caption{On the left side, the original image $X(x,0)$. On the right side
$X(x,3)$. Taken from \cite{Zambrano-Zuniga-2}.}%
\label{Figure 11}%
\end{center}
\end{figure}
%EndExpansion

In \cite{Zambrano-Zuniga-2} was proposed the following reaction-diffusion CNN
for denoising grayscale images polluted with normal additive
noise:{\footnotesize
\begin{equation}
\frac{\partial X(x,t)}{\partial t}=3X(x,t)+(\lambda I-\boldsymbol{D}%
_{0}^{\alpha})X(x,t)+3B\ast\left[  X_{0}(x)-f(X(x,t))\right]  , \label{CNN_5}%
\end{equation}
} where $\alpha=0.75$, $f(x)=0.5(|x+1|-|x-1|)$, $B(x)=(\Omega(p^{2}%
|x|_{p})-\Omega(|x|_{p}))$, and $-1\leq X_{0}(x)\leq1$. Notice that we are
using the interval $\left[  -1,1\right]  $ as a grayscale scale. This equation
was found experimentally. Natively, the reaction term $3X(x,t)+3B\ast\left[
X_{0}(x)-f(X(x,t))\right]  $ gives an estimation of the edges of the image,
while the diffusion term $(\lambda I-\boldsymbol{D}_{0}^{\alpha})X(x,t)$
produces a smoothed version of the image.

The processing of an image $X_{0}(x)$ using (\ref{CNN_5}) requires solving the
corresponding Cauchy problem with initial datum $X(x,0)=X_{0}(x)$. For an
in-depth discussion on the numerical and computational techniques required,
the reader may consult \cite{Zambrano-Zuniga-2}. See Figure \ref{Figure 12}.%

%TCIMACRO{\FRAME{ftbpFU}{4.875in}{2.4664in}{0pt}{\Qcb{Left side, the original
%image. Right side, the image plus Gaussian noise, mean zero and variance
%$0.05$. Taken from \cite{Zambrano-Zuniga-2}.}}{\Qlb{Figure 12}}{Figure
%12}{\special{ language "Scientific Word";  type "GRAPHIC";
%maintain-aspect-ratio TRUE;  display "USEDEF";  valid_file "T";
%width 4.875in;  height 2.4664in;  depth 0pt;  original-width 6.493in;
%original-height 3.2707in;  cropleft "0";  croptop "1";  cropright "1";
%cropbottom "0";  tempfilename 'SEIF1B03.wmf';tempfile-properties "XPR";}} }%
%BeginExpansion
\begin{figure}
[h]
\begin{center}
\includegraphics[width=0.75\textwidth]%
{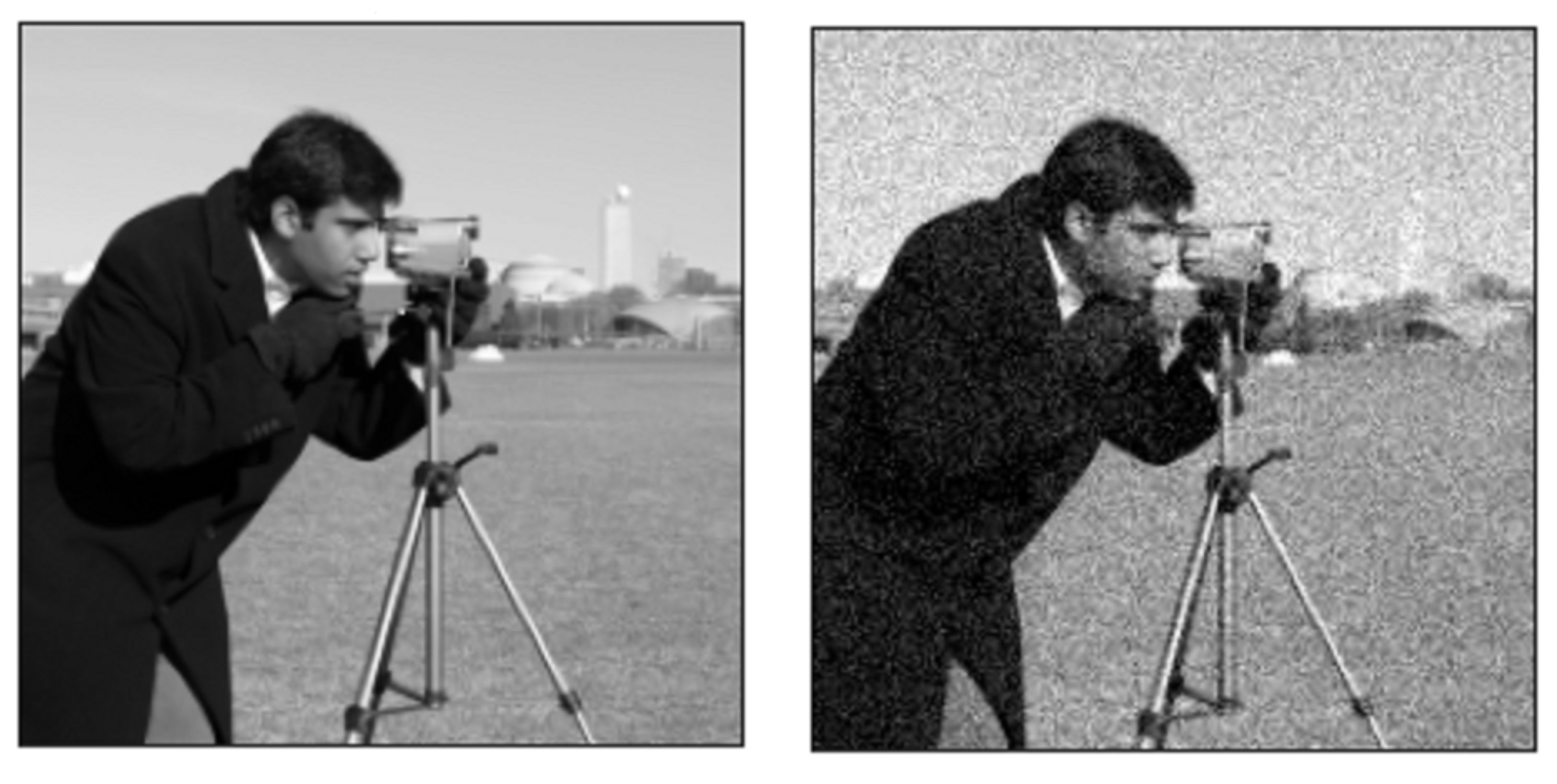}%
\caption{Left side, the original image. Right side, the image plus Gaussian
noise, mean zero and variance $0.05$. Taken from \cite{Zambrano-Zuniga-2}.}%
\label{Figure 12}%
\end{center}
\end{figure}
%EndExpansion
%TCIMACRO{\FRAME{ftbpFU}{4.7591in}{2.2943in}{0pt}{\Qcb{Left side, filtered
%image using Equation \ref{CNN_5}. Right side, filtered image obtained by using
%Perona-Malik equation with $\lambda=0.04$, $\delta_{t}=0.075$, and \ $t=100$
%iterations, and $g_{1}(s)$, see \cite{B-L-Mathm-image}. Taken from
%\cite{Zambrano-Zuniga-2}.}}{\Qlb{Figure 13}}{Figure 13}%
%maintain-aspect-ratio TRUE;  display "USEDEF";  valid_file "T";
%width 4.7591in;  height 2.2943in;  depth 0pt;  original-width 6.743in;
%original-height 3.2361in;  cropleft "0";  croptop "1";  cropright "1";
%cropbottom "0";  tempfilename 'SEIF5404.wmf';tempfile-properties "XPR";}} }%
%BeginExpansion
\begin{figure}
[h]
\begin{center}
\includegraphics[width=0.75\textwidth]%
{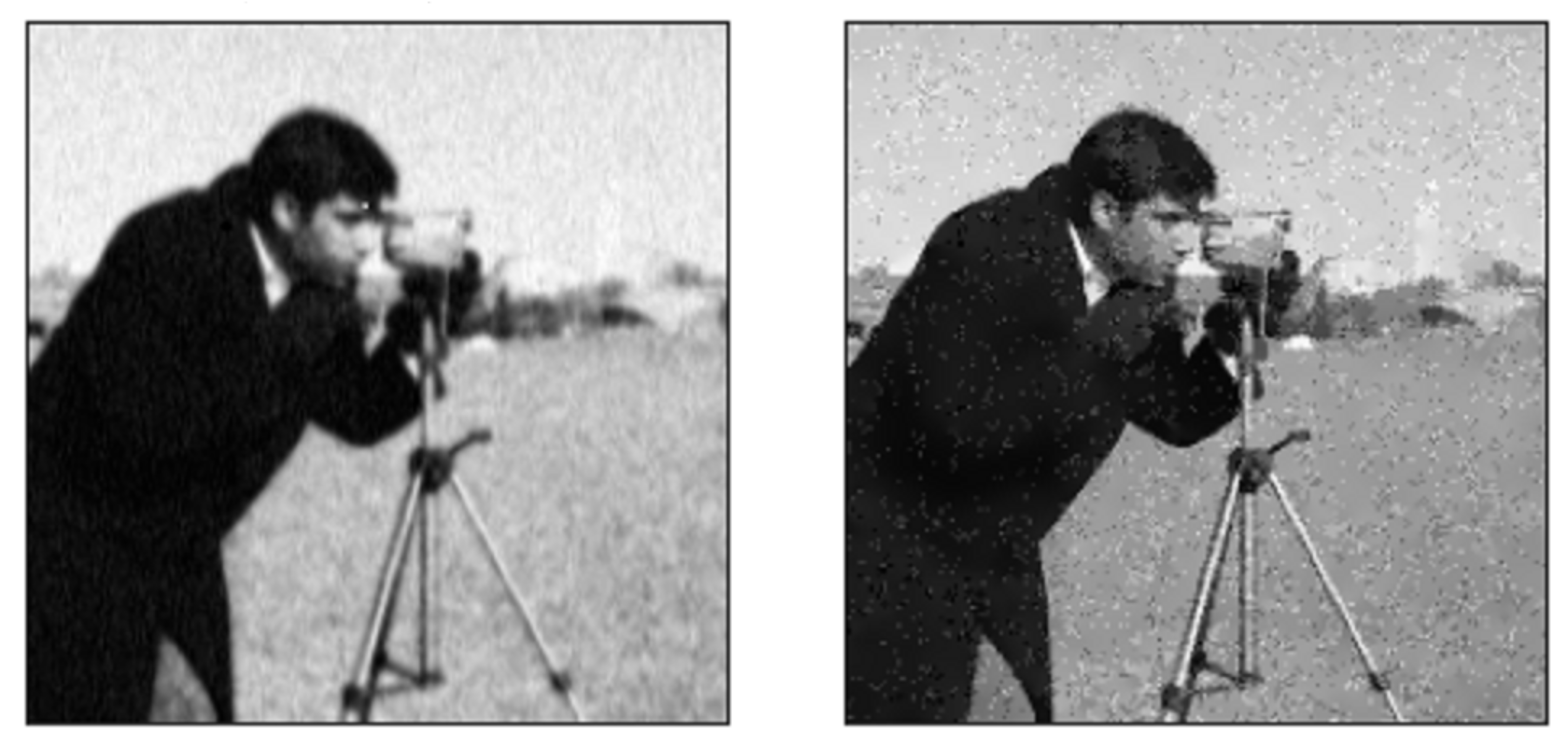}%
\caption{Left side, filtered image using Equation \ref{CNN_5}. Right side,
filtered image obtained by using Perona-Malik equation with $\lambda=0.04$,
$\delta_{t}=0.075$, and \ $t=100$ iterations, and $g_{1}(s)$, see
\cite{B-L-Mathm-image}. Taken from \cite{Zambrano-Zuniga-2}.}%
\label{Figure 13}%
\end{center}
\end{figure}
%EndExpansion

\section{$p$-Adic Wilson-Cowan Models and Connection Matrices}

In \cite{Zuniga-Entropy}, the Wilson-Cowan models were formulated on Abelian,
locally compact topological groups, see model \ref{Model_3}. The classical
model corresponds to the group $(\mathbb{R},+)$. Using classical techniques on
semilinear evolution equations, see e.g. \cite{Wu}-\cite{Milan}, it was showed
that the corresponding Cauchy problem is locally well-posed, and if
$r_{E}=r_{I}=0$, it is globally well-posed, see \cite[Theorem 1]%
{Zuniga-Entropy}. This last condition corresponds to the case of two coupled perceptrons.

We already discussed in Section \ref{Section_Neuwwons_geometry} that the
compatibility of the Wilson-Cowan\ model with the small-world network property
requires a non-negligible interaction between any two groups of neurons, in
turn, this fact requires that the neurons be organized in compact group. Now,
$\left(  \mathbb{R}^{N},+\right)  $ does not have non-trivial compact
subgroups. Indeed, if $x_{0}\neq0$, then $\left\langle x_{0}\right\rangle
=\left\{  nx_{0};n\in\mathbb{Z}\right\}  $ is a non-compact subgroup of
$\left(  \mathbb{R}^{N},+\right)  $, because $\left\{  \left\vert n\right\vert
;n\in\mathbb{Z}\right\}  $ is not bounded. This last assertion is equivalent
to the Archimedean axiom of the real numbers. In conclusion, the compatibility
between the Wilson-Cowan model and the small-world property requires changing
$\left(  \mathbb{R},+\right)  $ to a compact Abelian group. In
\cite{Zuniga-Entropy}, the authors selected the field of the $p$-adic numbers.
This field has infinitely many compact subgroups, the balls with center at the
origin. They selected the unit ball, the ring of $p$-adic numbers
$\mathbb{Z}_{p}$.

The $p$-adic Wilson-Cowan model admits good discretizations. Each
discretization corresponds to a system of non-linear integrodifferential
equations on a finite rooted tree. They show that the solution of the Cauchy
problem of this discrete system provides a good approximation to the solution
of the Cauchy problem of the $p$-adic Wilson-Cowan model, see \cite[Theorem
2]{Zuniga-Entropy}. They provide extensive numerical simulations of $p$-adic
Wilson-Cowan models that show that the $p$-adic models provide a similar
explanation to the numerical experiments presented in \cite{Wilson-Cowan-2}.
They also show that the connection matrices can be incorporated into the
$p$-adic Wilson-Cowan model.

\section{Some open problems}

\subsection{Traveling waves for neural networks}

The existence and uniqueness of traveling waves for model \ref{Model_0} have
been the focus of great interest; see, e.g., \cite{Ermentrout et
al1}-\cite{Ermentrout et al 2}, and the references therein. A such solution
has the form $u(x,t)=v(x-ct)$, where $c$ is the speed of the wave. In the
$p$-adic framework, studying such solutions is an open problem.
Time-independent traveling waves (called bumps) are localized solutions. This
type of solution has also been studied intensively; see, e.g.,
\cite{Kishimoto-Amari}-\cite{Oleynik et al}, and the references therein. The
study of this type of solution is also an open problem in the $p$-adic framework.

\subsection{Fuzzy CNNs}

In \cite{Yang et al}-\cite{Yang et al 2}, the fuzzy cellular neural networks
(FCNNs) were introduced as an extension of the CNNs. The FCNNs use fuzzy
operators in the synaptic law that combine the low-level information
processing capability of the CNNs with the high-level information processing
capability, such as image understanding, of fuzzy systems. The FCNNs has been
used in many different applications, for instance, in edge detection for
noised images \cite{Yang et al 2}-\cite{Doan et al}, segmentation of medical
images \cite{Shitong et al}-\cite{Shitong et al 2}, and image encryption, see
\cite{Ratnavelu et al}-\cite{Mani et al}, among several applications.

We propose the following formulation for the $p$-adic version of the FCNNs:%
\begin{align*}
\frac{\partial X(x,t)}{\partial t}  &  =-aX(x,t)+A\left(  x\right)  \ast
f((X(x,t))+B\left(  x\right)  \ast U(x)+Z(x)\\
&  +\bigvee_{z}\left(  A_{max}\left(  z\right)  \ast f(X(z,t))\Omega
(p^{l(x)}|z-x|_{p})\right) \\
&  +\bigvee_{z}\left(  B_{max}\left(  z\right)  \ast U(z))\Omega
(p^{l(x)}|z-x|_{p})\right) \\
&  +\bigwedge_{z}\left(  A_{min}\left(  z\right)  \ast f(X(z,t))\Omega
(p^{l(x)}|z-x|_{p})\right) \\
&  +\bigwedge_{z}\left(  B_{min}\left(  z\right)  \ast U(z))\Omega
(p^{l(x)}|z-x|_{p})\right)  ,
\end{align*}
where $X(x,t)$ and $f(X(x,t))$ are the state and output of cell $x\in
\mathbb{Z}_{p}$ at time $t$, respectively; $f:\mathbb{R}\rightarrow\mathbb{R}$
is a bounded Lipschitz function satisfying $f(0)=0$; $U(x)$ and $Z(x)$ are the
input and threshold of cell $x\in\mathbb{Z}_{p}$; $A,B\in\mathcal{C}%
(\mathbb{Z}_{p})$ are the feedback and feedforward kernels. Finally, $A_{max}%
$, $B_{max}$, $A_{min}$, and $B_{min}$ are kernels for the local fuzzy
operators $\bigvee_{z}$ and $\bigwedge_{z}$. We introduce a continuous
function $l:\mathbb{Z}_{p}\rightarrow\mathbb{N}$ that defines the locality of
the fuzzy operators. The Fuzzy operators are defined as%
\begin{align*}
\bigvee_{z}\phi(z)\Omega(p^{l(x)}|z-x|_{p})  &  :=\max\{\phi(z);\;z\in
x+p^{l(x)}\mathbb{Z}_{p}\},\\
\bigwedge_{z}\phi(z)\Omega(p^{l(x)}|z-x|_{p})  &  :=\min\{\phi(z);\;z\in
x+p^{l(x)}\mathbb{Z}_{p}\}.
\end{align*}
There are several open problems related with $p$-adic FCNNs, we just mention
two. First, to study the stability and the stationary states of the of
$p$-adic FCNNs. Second, to develop image processsing algorithms based on
$p$-adic FCNNs.

\subsection{Chaos in CNNs with delay}

The numerical experiments presented in Section \ref{Section_CNNS with delay}
show that the responses of the $p$-adic CNNs\ with delay have a chaotic
behavior. A relevant problem is investigating the different dynamical
behaviors, including limit cycles, chaos, etc., for different ranges of
parameters of the network.

\subsection{Development of visual computing techniques based on $p$-adic CNNs}

In \cite{Zambrano-Zuniga-1}-\cite{Zambrano-Zuniga-2}, the first two authors
showed that $p$-adic CNNs can be used to implement edge detection and
denoising algorithms. But, it is necessary to investigate the use of $p$-adic
CNNs in general visual computing tasks, see, e.g., \cite{Chua-Tamas},
\cite{Itoh et al}.

\end{document}